%% file: main.tex
\documentclass{article}

\usepackage{microtype}
\usepackage{graphicx}
\usepackage{subfig}
\usepackage{enumitem}
\usepackage{color}
\usepackage[table]{xcolor}
\usepackage{booktabs} 
\usepackage{lipsum}
\usepackage{multirow}
\usepackage[symbol]{footmisc}

\RequirePackage{amsthm,amssymb}

\usepackage{hyperref}

\input{math_commands.tex}
\input{my_commands.tex}

\PassOptionsToPackage{numbers, compress}{natbib}

\usepackage[preprint]{neurips_2020}

\begin{document}

\title{A Collective Learning Framework to\\ Boost GNN Expressiveness}

\author{Mengyue Hang \\
Department of Computer Science \\
Purdue University \\
\texttt{hangm@purdue.edu}
\And
Jennifer Neville \\
Department of Computer Science \\
Purdue University \\
\texttt{neville@cs.purdue.edu}\\
\And
Bruno Ribeiro \\
Department of Computer Science \\
Purdue University \\
\texttt{ribeiro@cs.purdue.edu}
}

\maketitle

\begin{abstract}
Graph Neural Networks (GNNs) have recently been used for node and graph classification tasks with great success.
Unfortunately, existing GNNs are not universal (i.e., most-expressive) graph representations.
In this work, we propose {\em collective learning}, a general collective classification Monte Carlo approach for graph representation learning that provably increases the representation power of existing GNNs. 
We show that our use of Monte Carlo sampling is key in these results.
Our experiments consider the task of inductive node classification across partially-labeled graphs using five real-world network datasets and demonstrate a consistent, significant boost in node classification accuracy when our framework is combined with a variety of state-of-the-art GNNs.

\end{abstract}

\input{1_intro.tex}
\input{2_problemdef}

\input{3_method.tex}

\input{4_theory}
\input{5_experiments.tex}
\input{6_relatedwork}

\input{7_conclusion.tex}

\input{8_impact}


\bibliography{reference}
\bibliographystyle{icml2020}

\clearpage

 \input{9_appendix}

\end{document}

%% file: 1_intro.tex
\section{Introduction}

Graph Neural Networks (GNNs) have recently shown great success at node and graph classification tasks \citep{hamilton2017inductive,kipf2016semi,luan2019break,xu2018powerful}. GNNs have been applied in both transductive settings (where the test nodes are embedded in the training graph) and inductive settings (where the training and test graphs are disjoint). 
Despite their success, existing GNNs are no more powerful than the Weisfeiler-Lehman (WL) graph isomorphism test, and thus, inherit its shortcomings, i.e. they are not universal (most-expressive) graph representations~\cite{chen2019equivalence,morris2019weisfeiler,murphy2019relational,xu2018powerful}.
In other words, these GNNs (which we refer to as WL-GNNs and also include GCNs~\cite{kipf2016semi}) are not expressive enough for some node classification tasks, since their representation can provably fail to distinguish non-isomorphic nodes with different labels.

At the same time, a large body of work in relational learning has focused on strengthening poorly-expressive (i.e., local) classifiers in relational models (e.g., relational logistic regression, naive Bayes, decision trees \citep{neville2003learning}) in {\em collective classification} frameworks, by incorporating dependencies among node labels and propagating inferences during classification to improve performance, particularly in semi-supervised settings~\citep{koller2007introduction,pfeiffer2015overcoming,xiang2008pseudolikelihood}.

In this work, we theoretically and empirically investigate the hypothesis that, by explicitly incorporating label dependencies among neighboring nodes via predicted label sampling---akin to how collective classification improves not-so-expressive classifiers---it is possible to devise an add-on training and inference procedure that can improve the expressiveness of any existing WL-GNN for inductive node classification tasks, which we denote {\em collective learning}. 

{\bf Contributions.} We first show that collective classification is provably unnecessary if one can obtain GNNs that are most-expressive. 
Then, because current WL-GNNs are {\em not} most-expressive, we propose an add-on general collective learning framework to GNNs that provably boosts their expressiveness, beyond that of an {\em optimal}  WL-GNN\footnote{We use the term optimal WL-GNN to refer to the most expressive version of a GNN--one that has the same distinguishing power as a Weisfeiler-Lehman test. Note this is not a universal graph representation.}. Our framework, which we call \ourmodel, involves the use of self-supervised learning and sampled embeddings to incorporate node labels during inductive learning---and it can be implemented with {\em any} component GNN. 
In addition to being strictly more expressive than optimal WL-GNNs, we also show how collective learning improves finite $d$-layer WL-GNNs in practice by extending their power to distinguish non-isomorphic nodes from $d-$hop neighborhoods to $2d$.
We also show that, in contrast to our proposed framework, attempts to incorporate collective classification ideas into WL-GNNs without sampled embeddings (e.g., \cite{fan2019recurrent,qu2019gmnn,vijayan2018hopf}) cannot increase expressivity beyond that of an optimal WL-GNN.  
Our empirical evaluation shows \ourmodel achieves a consistent improvement of node classification accuracy, across a variety of state-of-the-art WL-GNNs, for tasks involving unlabeled and partially-labeled test graphs.



%% file: 2_problemdef.tex
\vspace{-2mm}
\section{Problem formulation}
\vspace{-1mm}

We consider the problem of \textit{inductive} node classification across partially-labeled graphs, which takes as input a graph $G^\text{(tr)} = (V^\text{(tr)}, E^\text{(tr)}, \mX^\text{(tr)}, \mY^\text{(tr)}_L )$ for training, where $V^\text{(tr)}$ is a set of $n^\text{(tr)}$ vertices, $E^\text{(tr)} \subset V^\text{(tr)} \times V^\text{(tr)}$ is a set of edges with adjacency matrix $\mA^\text{(tr)}$, $\mX^\text{(tr)}$ is a $n^\text{(tr)} \times p$ matrix containing node attributes as $p$-dimensional vectors, and $\mY^\text{(tr)}_L$ is a set of observed labels (with $C$ classes) of a connected set of nodes $V^\text{(tr)}_L \subset V^\text{(tr)}$, where $V^\text{(tr)}_L$ is assumed to be a proper subset of $V^\text{(tr)}$, noting that $V^\text{(tr)}_L \neq \emptyset$.
Let $\mY_U^\text{(tr)}$ be the unknown labels of nodes $V^\text{(tr)}_U = V^\text{(tr)} \setminus V^\text{(tr)}_L$.
The goal is to learn a {\em joint} model of $\mY_U^\text{(tr)} \sim P(\mY_U | G^\text{(tr)})$ and apply this same model to predict hidden labels $\mY_U^\text{(te)}$ in another test graph $G^\text{(te)}$, i.e., $\hat{\mY}_U^\text{(te)} = \argmax_{\mY_U} P(\mY_U | G^\text{(te)})$. The test graph $G^\text{(te)}$ can be partially labeled or unlabeled so $V^\text{(te)}_L \supseteq \emptyset$.

Graph Neural Networks (GNNs), which aggregate node attribute information to produce node representations, have been successfully used for this task. At the same time, relational machine learning (RML) methods, which use collective inference to boost the performance of local node classifiers via (predicted) label dependencies, have also been successfully applied to this task. 

Since state-of-the-art GNNs are not most-expressive \cite{morris2019weisfeiler,murphy2019relational,xu2018powerful}, collective classification ideas may help to improve the expressiveness of GNNs. In particular, collective inference methods often {\em sample} predicted labels (conditioned on observed labels) to improve the local representation around nodes and approximate the joint distribution $P(\mY_U | G^\text{(te)})$. We also know from recent research that sampling randomized features can boost GNN expressiveness \cite{murphy2019relational}.    
This leads to the key conjecture of this work \Cref{hypo: more_expressive}, which we prove theoretically in \Cref{sec: theory} and validate empirically by extensive experimentation in \Cref{sec:results}.

\begin{hypothesis}
\label{hypo: more_expressive}
Since current Graph Neural Networks (e.g. GCN, GraphSAGE, TK-GNN) cannot produce most expressive graph representations, collective learning (which takes label dependencies into account via Monte Carlo sampling) can improve the accuracy of node classification by producing a more expressive graph representation. 
\end{hypothesis}

Why? Because WL-GNNs can extract more information about local neighborhood dependencies via sampling~\cite{murphy2019relational}, and sampling predicted labels allows the GNN to pay attention to the relationship between node attributes, the graph topology, and label dependencies in local neighborhoods. 
With {\em collective learning}, GNNs will be able to incorporate more information into the estimated joint label distribution. 
%
%
Next, we describe our {\em collective learning} framework.

%% file: 3_method.tex
\section{Proposed framework: {\em Collective Learning} to boost GNNs}
\label{sec:framework}



In this section, we outline \ourmodel. It is a general framework to incorporate any GNN, and combines {\em self-supervised learning approach} and {\em Monte Carlo embedding sampling} in an {\em iterative} process to improve inductive learning on partially labeled graphs.


Specifically, given a partially labeled training graph $G^\text{(tr)} = (V^\text{(tr)}, E^\text{(tr)}, \mX^\text{(tr)}, \mY_L^\text{(tr)})$ with adjacency matrix $\mA^\text{(tr)}$ and a partially-labeled test graph $G^\text{(te)} = (V^\text{(te)}, E^\text{(te)}, \mX^\text{(te)}, \mY_L^\text{(te)})$ with adjacency matrix $\mA^\text{(te)}$. 
The goal of inductive node classification task is to train a joint model on $G^\text{(tr)}$ to learn $P(\mY_U|G^\text{(tr)})$ 
and apply it to $G^\text{(te)}$ by replacing the input graph $G^\text{(tr)}$ with $G^\text{(te)}$.
Suppose the graphs $G^\text{(tr)}$ and $G^\text{(te)}$, we can define $\mY^\text{(tr)}_L$ as a binary (0-1) matrix of dimension $|V^\text{(tr)}|\times C$, and $\mY_L^\text{(te)}$ of dimension $|V^\text{(te)}|\times C$, where the rows corresponding to the one-hot encoding of the (available) labels.

{\bf (Background) GNN and representation learning.}
Given a partially labeled graphs $G^\text{(tr)}$, WL-GNNs generate node representation by propagating feature information throughout the graph. Specifically, $\forall v \in V^\text{(tr)},$ 
\begin{align}
P(\mY_v|\mX^\text{(tr)}, \mY^\text{(tr)}_L, \mA^\text{(tr)})  = \sigma({\bf W}\mZ_v + {\bf b}), \label{equ:gnn_obj}
\end{align}
where $\mZ_v = \text{GNN}(\mX^\text{(tr)}, \mA^\text{(tr)}; \Theta)_v $ is the GNN representation of node $v$, $\sigma(\cdot)$ is the softmax activation, and $\Theta$, ${\bf W}$ and ${\bf b}$ are model parameters, which are learned by minimizing the cross-entropy loss between true labels $\mY^\text{(tr)}_L$ and the predicted labels.

{\bf The collective learning framework.} Following \Cref{hypo: more_expressive}, we propose Collective Learning GNNs (\ourmodel), which includes label information as input to GNNs to produce a more expressive representation. 
The overall framework follows three steps: ({\bf Step 1}) Include labels in the input using a random mask; ({\bf Step 2}) Sample predicted labels for whatever is masked, use as input to WL-GNN that also takes all node labels as input, and average representations of the WL-GNN over the sampled predicted labels; ({\bf Step 3}) Perform one optimization step by minimizing a negative log-likelihood upper bound (per \Cref{prop:opt}). Collective learning for WL-GNNs then consists of iterating over Steps 1-3 for $t=1,\ldots,T$ iterations. Finally, once optimized, we perform inference via Monte Carlo estimates.



{\bf Step 1. Random masking and self supervised learning.} 
The input to GNNs is typically the full graph $G^\text{(tr)}$. If we included the observed labels $\mY^\text{(tr)}_L$ directly in the input, then it would be trivial to learn a model that predicts part of the input. Instead, we either ({\em scenario test-unlabeled}) mask all label inputs if the test graph $G^\text{(te)}$ is expected to be unlabeled; or ({\em scenario test-partial}) if $G^\text{(te)}$ is expected to have partial labels, we apply a mask to the labels we wish to predict in training so they do not appear in the input $\mY^\text{(tr)}_L$. 

In ({\em scenario test-partial}), where $G^\text{(te)}$ is expected to have some observed labels, at each stochastic gradient descent step, we randomly sample a binary mask $\mM \sim \text{Uniform}(\cM)$ from a set of masks, where $\mM$ is a $|V^\text{(tr)}|\times C$ binary (0-1) matrix with the same $|V^\text{(tr)}|$-dimensional vector in each column.
By applying the mask on the observed labels $\mY_L^\text{(tr)}$, the set of true labels is effectively partitioned into two parts, where part of true labels $\mY_L^\text{(tr)} \odot \mM$ are used as input to \ourmodel, and the other part $\mY_L^\text{(tr)} \odot \overline{\mM}$ are used as optimization target, where $\overline{\mM} := {\bf 1} - \mM$ is the bitwise negated matrix of $\mM$. 

{\bf Step 2. \ourmodel loss and its representation averaging.}
At the $t$-th step of our optimization ---these steps can be coarser than a gradient step--- we either ({\em scenario test-partial}) sample a mask $\mM^{(t)} \sim \text{Uniform}(\cM)$ or ({\em scenario test-unlabeled}) set $\mM^{(t)} = {\bf 0}$.
For now, we
assume we can sample $\hat{\bf Y}^{(t-1)}=(\hat{\mY}_v^{(t-1)})_{v \in V^\text{(tr)}}$ from an estimate of the distribution $P(\mY_v^\text{(tr)}|\mX^\text{(tr)}, \mY_L^\text{(tr)} \odot \mM^{(t)}, \mA^\text{(tr)})$ ---we will come back to this assumption soon. 
Let $\mX_{\mY_L^\text{(tr)},\hat{\mY}^{(t-1)},\mM^{(t)}}^\text{(tr)}$ be the matrix  concatenation between $\mY^{\text{(tr)}}_{L}\odot \mM^{(t)} + \hat{\mY}^{(t-1)}\odot \overline{\mM}^{(t)}$ and $\mX^\text{(tr)}$, where again $\overline{\mM} := {\bf 1} - \mM$ is the bitwise negated matrix of $\mM$.
Let
\begin{equation}\label{eq:Z_labeled}
\begin{split}
\mZ_v^{(t)}(\mM^{(t)};\Theta)= \expectation_{\hat{\mY}^{(t-1)}} \big[ \text{GNN}(\mX_{\mY_L^\text{(tr)},\hat{\mY}^{(t-1)},\mM^{(t)}}^\text{(tr)} , \mA^\text{(tr)} ; \Theta)_v \big],
\end{split}
\end{equation}
where GNN represents an arbitrary graph neural network model and $Z_v^t$ is the \ourmodel's representation obtained for node $v \in V^\text{(tr)}$ at step $t \geq 1$. 

Our optimization is defined over the expectation of $\mZ_v^{(t)}(\mM^{(t)})$ w.r.t.\ to the sampled predicted labels $\hat{\mY}^{(t-1)}$ (\Cref{eq:Z_labeled}) and over a loss averaged over all sampled masks (noting that the case where $\mM^{(t)} = {\bf 0}$ is trivial):
\begin{equation}
\label{eq:mcgnn_obj_mask}
\begin{split}
   \Theta_{t},&{\bf W}_{t},{\bf b}_{t} 
   = \argmin_{\Theta,{\bf W},{\bf b}} \mathbb{E}_{\mM^{(t)}} \bigg[-{\sum_{v\in V_L^\text{(tr)}}\overline{\mM}^{(t)}_v \log  \sigma({\bf W}\mZ_v^{(t)}(\mM^{(t)};\Theta) + {\bf b})_{y^\text{(tr)}_v}} \bigg],
\end{split}
\end{equation}
where again, $\sigma(\cdot)$ is the softmax activation function, and $V_L^\text{(tr)}$ are the labeled nodes in training graph.

{\bf Obtaining $\hat{\bf Y}^{(t-1)}$.}
At iteration $t$, we use the learned \ourmodel model parameter $\Theta_{t-1}$ to obtain $\mZ_v^{(t-1)}$ according to \Cref{eq:Z_labeled} and use the \ourmodel model parameters ${\bf W}_{t-1},{\bf b}_{t-1}$ to obtain the label prediction recursively
\begin{equation}\label{eq:yhat}
    \hat{\mY}^{(t-1)}_v \sim \text{Categorical}(\sigma({\bf W}_{t-1}\mZ^{(t-1)}_v(\mM^{(t)}
    ;\Theta_{t-1}) + {\bf b}_{t-1})), \; v \in V^\text{(tr)},
\end{equation}

starting the recursion with $\hat{\mY}^{(t-2)} = {\bf 0}$. 

{\bf Step 3. Stochastic optimization of \Cref{eq:mcgnn_obj_mask}.}
\Cref{eq:mcgnn_obj_mask} is based on a pseudolikelihood, where the joint distribution of the labels $\{\mY_v^\text{(tr)} : v \in V_L^\text{(tr)} \text{ s.t. } \overline{\mM}^{(t)}_v = 1\}$ is decomposed as marginal distributions resulting in the sum over $V_L^\text{(tr)}$ in \Cref{eq:mcgnn_obj_mask}.
In order to optimize \Cref{eq:mcgnn_obj_mask}, we compute gradient estimates w.r.t.\ $\Theta$ and ${\bf b}$ using the
following sampling procedure.

1.\ We first need to compute an unbiased estimate of $\{\mZ^{(t-1)}_v\}_{v \in V_L^\text{(tr)}}$ in \Cref{eq:Z_labeled} using $K$ i.i.d.\ samples $\hat{\mY}^{(t-1)}$ from the model obtained at time step $t-1$ (as describe above).
Note that the time/space complexity of the \ourmodel is $K$ times the time/space complexity of the corresponding GNN model as we have to compute $K$ representations for each node at each stochastic gradient step. 

2.\ Next, we need an unbiased estimate of the expectation over $\mM^{(t)}$ in \Cref{eq:mcgnn_obj_mask}.
In {\em (scenario test-partial)} the unbiased estimates are obtained by sampling $\mM^{(t)} \sim \text{Uniform}(\cM)$, in the {\em (scenario test-unlabeled)} the value obtained is exact since $\mM^{(t)} = {\bf 0}$.
\Cref{prop:opt} shows that the above procedure is a proper surrogate upperbound of the loss function

{\bf Inference with learned model.}
We apply the same procedure as in {\em obtaining $\hat{\bf Y}^{(t-1)}$}, but transferring the learned parameters to a different test graph.
Specifically, at iteration $t$, \ourmodel parameters $\Theta_t, \mW_t, \vb_t$ are learned according to \Cref{eq:mcgnn_obj_mask} on the training graph $G^\text{(tr)}$. Suppose the iteration contains $J$ gradient steps, thus $J$ sampled masks are used. Given an any-size attributed graph $G^{(te)}$, we sample another $J$ masks $\mM^{(t)}$ of size $|V^\text{(te)}|$, either ({\em scenario test-partial}) sampling $\mM^{(t)} \sim \text{Uniform}(\cM)$ or ({\em scenario test-unlabeled}) set $\mM^{(t)} = {\bf 0}$. We also sample $K$ predicted labels $\{\hat{\bf Y}^{\text{(te)}, (t-1)}_1, \cdots, \hat{\bf Y}^{\text{(te)}, (t-1)}_K\}$ from the predicted probability distribution where $\hat{\bf Y}^{\text{(te)}, (t-1)}_k=(\hat{\mY}_v^{(t-1)})_{v \in V^{(te)}}$ 
. The node representations for $v \in V^\text{(te)}$ are obtained using $\mM^{(t)}$ and $\hat{\bf Y}^{\text{(te)}, (t-1)}_{1,\cdots,K}$:  
\begin{equation}\label{eq:Z_labeled_test}
\mZ_v^{(t)} = \frac{1}{J K} \sum_{j=1}^J \sum_{k=1}^K \text{GNN}(\mX^\text{(te)}_{\mY_L^\text{(te)},\hat{\mY}_k^{\text{(te)},(t-1)},\mM^{(t)}_j} , \mA^\text{(te)} ; \Theta_t)_v , \; \forall v \in V^\text{(te)},
\end{equation}
where again, $J$ is the number of gradient steps per iteration $t$ and $K$ is the number of Monte Carlo samples of $\hat{\bf Y}^{\text{(te)}, (t-1)}$.
Then the label predictions are obtained using the learned \ourmodel parameters $\mW_t, \vb_t$:
\[
\hat{\mY}_v^{(t)} \sim \text{Categorical}(\sigma({\bf W}_{t}\mZ_v^{(t)} + {\bf b}_{t})_v), \; \forall v \in V^\text{(te)}.
\]
Note that the test label predictions $\hat{\mY}^{\text{(te)},(t)}$ are also recursively updated, and the recursion starts with $\hat{\mY}^{\text{(te)},(t-2)} = {\bf 0}$. 




%% file: 4_theory.tex
\section{Collective learning analysis} \label{sec: theory}

{\em Is collective classification able to better represent target label distributions than graph representation learning?}
The answer to this question is both {\em yes (for WL-GNNs) and no (for most-expressive representations)}.
\Cref{thm:col} shows that a most-expressive graph representation~\citep{murphy2019relational,maron2019universality, srinivasan2019equivalence} would not benefit from a collective learning boost.
All proofs can be found in the \Appendix.
%
%
\begin{restatable}[Collective classification can be unnecessary]{theorem}{thmCCun}
\label{thm:col}
Consider the task of predicting node labels when no labels are available in test data.
Let $\Gamma^\star(v,G^\text{(te)})$ be a most-expressive representation of node $v \in V^\text{(te)}$ in graph $G^\text{(te)}$ . 
Then, for any collective learning procedure predicting the class label of $v \in V^\text{(te)}$, there exists a classifier that takes $\Gamma^\star(v,G^\text{(te)})$ as input and predicts the label of $v$ with equal or higher accuracy.
\end{restatable}
%
%
While \Cref{thm:col} shows that the most-expressive graph representation does not need collective classification, WL-GNNs are not most-expressive~\cite{morris2019weisfeiler,murphy2019relational,xu2018powerful}. Indeed, \Cref{thm:cl-power} and \Cref{prop:expand-power} show that  \ourmodel boosts the expressiveness of optimal WL-GNN and practical WL-GNNs, respectivelly.
Then, we show that the stochastic optimization in Step 3 optimizes a loss surrogate upper bound.

\subsection{Expressive power of \ourmodel}
\citet{morris2019weisfeiler} and \citet{xu2018powerful} show that WL-GNNs are no more powerful in distinguishing non-isomorphic graphs and nodes as the standard Weisfeiler-Lehman graph isomorphism test (1-WL or just WL test). 
Two nodes are assumed isomorphic by the WL test if they have the same color assignment in the stable coloring. 
%
The node-expressivity of a parameterized graph representation $\Gamma$ (with parameter $\Gamma(\cdot;\mW)$) can then be determined by the set of graphs for which $\Gamma$ can identify non-isomorphic nodes:
\begin{equation*}
    \cG(\Gamma) =  \{  G :  \exists \mW^\star_G, \text{ s.t. } \forall u,v \in V_G,  \Gamma(G;\mW^\star_G)_v = \Gamma(G;\mW^\star_G)_u \text{ iff } u \text{, } v \text{ are isomorphic}, G \in \sG\},
\end{equation*}
where $\sG$ is the set of all any-size attributed graphs, $V_G$ is the set of nodes in graph $G$.
We call $\cG(\Gamma)$ the \textit{identifiable set} of graph representation $\Gamma$. 

The {\em most expressive} graph representation $\Gamma^\star$ has an {\em identifiable set} of all any-size attributed graphs, i.e. $\cG(\Gamma^\star) = \sG$. We refer to the WL-GNN that is equally expressive as WL test as the {\em optimal} WL-GNN (or $\text{WLGNN}^\star$), which is at least as expressive as all other WL-GNNs.

In this section we show that {\em collective learning} can boost the optimal $\text{WLGNN}^\star$, i.e., the identifiable set of $\text{WLGNN}^\star$ is a proper subset of collective learning over $\text{WLGNN}^\star$ (denoted $\text{\ourmodel}^\star$)
\[\cG(\text{WLGNN}^\star) \subsetneq 
\cG(\text{\ourmodel}^\star).\]
\begin{restatable}[\ourmodel{}$^\star$ expressive power]{theorem}{thmclpower}
\label{thm:cl-power}
Let $\text{WLGNN}^\star$ be an optimal WL-GNN.
Then, the collective learning representation of \Cref{eq:Z_labeled}, using $\text{WLGNN}^\star$ as the GNN component,  (denoted $\text{\ourmodel}^\star$) is strictly more expressive than this $\text{WLGNN}^\star$ representation model applied to the same tasks.
\end{restatable}

\Cref{thm:cl-power} answers \Cref{hypo: more_expressive}, by showing that by incorporating collective learning and sampling procedures, \ourmodel can boost the expressiveness of WL-GNNs, including the optimal $\text{WLGNN}^\star$. 

\begin{corollary} \label{coro:gnn}
Consider a graph representation learning method that, at iteration $t$, replaces $\hat{\mY}^{(t-1)}$, in \Cref{eq:yhat,eq:Z_labeled} with a deterministic function over $\mZ^{(t-1)}$, e.g., a softmax function that outputs $(P(\hat{\mY}^{(t-1)}_v| \mZ^{(t-1)}))_{v \in V^\text{(tr)}}$. 
Then, such method will be no more expressive than the optimal $\text{WLGNN}^\star$ and, hence, less expressive than $\text{\ourmodel}^\star$. 
\end{corollary}
\Cref{coro:gnn} shows that existing graph representation methods that ---on the surface--- may even look like \ourmodel, but do not perform the crucial step of sampling $(\hat{\mY}_v^{(t-1)})_{v \in V^\text{(tr)}}$, unfortunately, are no more expressive than  WL-GNNs. Examples of such methods include~\cite{fan2019recurrent,moore2017deep,qu2019gmnn,vijayan2018hopf}.

Next, we see that the practical benefits of collective learning are even greater when the WL-GNN has limited expressive power due to a constraint on the number of message-passing layers.

\subsection{How \ourmodel further expands the power of few-layer WL-GNNs}
A $d$-layer ($d > 1$) WL-GNN will only aggregate neighborhood information within $d$ hops of any given node (i.e., over a $d$-hop egonet, defined as the graph representing the connections among all nodes that are at most $d$ hops away from the center node).
In practice ---mostly for computational reasons--- WL-GNNs have many fewer layers than the graph's diameter $D$, i.e., $d < D$. %
%
For instance, GCN \cite{kipf2016semi}  and GraphSAGE \cite{hamilton2017inductive} both used $d = 2$ in their experiments.
Hence, they cannot differentiate two non-isomorphic nodes that are isomorphic within their $d$-hop neighborhood. 
We now show that \ourmodel can gather $2d$-hop neighborhood information with a $d$-layer WL-GNN. 

\begin{restatable}{proposition}{propexpandpower}
\label{prop:expand-power}
Let $G_{v}^d$ be the $d$-hop egonet of a node $v$ in graph $G$ with diameter $D > d$.
Let $v_1$ and $v_2$ be two non-isomorphic nodes whose $d$-hop egonets are isomorphic (i.e., $G_{v_1}^d$ is isomorphic to $G_{v_2}^d$) but $2d$-hop egonets are not isomorphic.
Then, a WL-GNN representation with $d$ layers will generate identical representations for $v_1$ and $v_2$ while \ourmodel is capable of giving distinct node representations.
\end{restatable}

\Cref{prop:expand-power} shows that collective learning has yet another benefit: \ourmodel further boosts the power of WL-GNNs with limited message-passing layers by gathering neighborhood information within a larger radius. 
Specifically, \ourmodel built on a WL-GNN with $d$ layers can enlarge the effective neighborhood radius from $d$ to $2d$ in \Cref{eq:Z_labeled} , while WL-GNN would have to stack $2d$ layers to achieve the same neighborhood radius, which in practice may cause optimization challenges (i.e., $d=2$ is a common hyperparameter value in the literature).


\vspace{-3pt}
\subsection{Optimization of \ourmodel}
\vspace{-5pt}
\begin{restatable}{proposition}{propopt}
\label{prop:opt}
If $\forall v \in V_L^\text{(tr)}$, $\nabla_{\Theta} ({\bf W}\mZ_v^{(t)}(\mM^{(t)};\Theta))_{y^\text{(tr)}_v}$ 
is bounded (e.g., via gradient clipping), then the optimization in \Cref{eq:mcgnn_obj_mask}, with the unbiased sampling of $\{\mZ^{(t-1)}_v\}_{v \in V^\text{(tr)}}$ and $\mM^{(t)}$ described above, results in a Robbins-Monro~\citep{robbins1951} stochastic optimization algorithm that optimizes a surrogate upper bound of the loss in \Cref{eq:mcgnn_obj_mask}.
\end{restatable}
%
%
Since the optimization objective in \Cref{eq:mcgnn_obj_mask} is computationally impractical, as it requires computing all possible binary masks and label predictions, \Cref{prop:opt} shows that the sampling procedures used in \ourmodel that considers $K$ samples of label predictions and a random mask at \break each gradient step is a feasible approach of estimating an unbiased upper bound of the objective.

%% file: 5_experiments.tex
\vspace{-2mm}
\section{Experiments}\label{sec:results}
\vspace{-1mm}


\subsection{Experiment Setup} \label{sec: exp}
{\bf Datasets.}
We use datasets of Cora, Pubmed, Friendster, Facebook, and Protein. The largest dataset (Friendster \cite{teixeira2019graph}) has 43,880 nodes, which is a social network of users where the node attributes include numerical features (e.g number of photos posted) and categorical features (e.g. gender, college, music interests, etc) encoded as binary one-hot features. The node labels represent one of the five age groups. Please refer to \Cref{appx:datasets} for more details.

{\bf Train/Test split.}
To conduct inductive learning tasks, for each dataset we split the graphs for training and testing, and the nodes are sampled to guarantee that there is no overlapping between any two sets of $ V_L^\text{(tr)}$, $ V_L^\text{(te)}$ and $V_T$. In our experiments, we tested two different label rates in test graph: $0$ (unlabeled) and $50\%$. We run five trials for all the experiments, and in each trial we randomly split the nodes/graphs for training/validation and testing. 

As our method can be applied to any GNN models, we tested three GNNs as examples:
\begin{itemize}[leftmargin=*]
\vspace{-2mm}
    \item GCN \cite{kipf2016semi} which includes two graph convolutional layers. Here we implemented an inductive variant of the original GCN model for our tasks. 
\vspace{-.5mm}
    \item Supervised GraphSage \cite{hamilton2017inductive} (denoted by GS) with Mean pooling aggregator. We use sample size of 5 for neighbor sampling.
\vspace{-.5mm}
    \item Truncated Krylov GCN \cite{luan2019break} (denoted by TK), a recent GNN model that leverages multi-scale information in different ways and are scalable in depth. The TK has stronger expressive power and achieved state-of-the-art performance on node classification tasks. We implemented Snowball architecture which achieved comparable performance with the other truncated Krylov architecture according to the original paper. 
\vspace{-2mm}
\end{itemize}

For each of GNNs, we compare its baseline performance (on its own) to the performance achieved using collective learning in \ourmodel. Note that we set the number of layers to be $2$ for GCN \cite{kipf2016semi} and GraphSage \cite{hamilton2017inductive} as set in their original papers, and use $10$ layers for TK \cite{luan2019break}. For a fair comparison, the baseline GNN and \ourmodel are trained with the same hyper-parameters, e.g. number of layers, hidden dimensions, learning rate, early-stopping procedures. Please refer to \Cref{appx:datasets} for details.

In addition, we also compare to three relational classifiers, ICA \cite{lu2003link}, PL-EM \cite{pfeiffer2015overcoming} and GMNN \cite{qu2019gmnn}. The first two models apply collective learning and inference with simple local classifiers —— Naive Bayes for PL-EM and Logistic regression for ICA. GMNN is the state-of-the-art collective model with GNNs, which uses two GCN models to model label dependency and node attribute dependency respectively. All the three models take true labels in their input, thus we use $\mY_L^\text{(tr)}$ for training and $\mY_L^\text{(te)}$ for testing. 

We report the average accuracy score and standard error of five trials for the baseline models, and compute the absolute improvement of accuracy of our method over the corresponding base GNN. We report the {\em balanced} accuracy scores on Friendster dataset as the labels are highly imbalanced. To evaluate the significance of our model's improvements, we performed a paired t-test with five trials.

\subsection{Results}
The node classification accuracy of all the models is shown in \Cref{tab:results_labeled}. Our proposed collective learning boost is denoted as +CL (for \textbf{C}ollective \textbf{L}earning) in the results and our model performance (absolute $\%$ of improvement over the corresponding baseline GNN) is shown in shaded area. Numbers in bold represent significant improvement over the baseline GNN based on a paired t-test ($p < 0.05$), and numbers with $^*$ is the best performing method in each column. 


{\bf Comparison with baseline GNN models.} \Cref{tab:results_labeled} shows that our method improves the corresponding non-collective GNN models for all the three model architectures (i.e. GCN, GraphSage and TK). Although all the models have large variances over multiple trials --- which is because different parts of the graphs are being trained and tested on in different trials, our model consistently improves the baseline GNN. The results from a paired t-test comparing the performance of our method and the corresponding non-collective GNN shows that the improvement is almost always significant (marked as bold), with only four exceptions. 
Comparing the gains on different datasets, our method achieved smaller gains on Friendster. This is because the Friendster graph is much more sparse than all other graphs (e.g. edge density of Friendster is 1.5e-4 and edge density of Cora is 1.44e-3 \cite{teixeira2019graph}), which makes it hard for any model to propagate label information and capture the dependency. 

Moreover, comparing the improvement over GCN and TK, we can observe that in general our method adds more gains to GCN performance. For example, with 3\% training labels on Cora, our method when combining with GCN has an average of 6.29\% improvement over GCN, 2.35\% improvement over GraphSage, and 0.96\% improvement over TK. This is in line with our assumption \Cref{hypo: more_expressive} that collective learning can help GNNs produce a more expressive representation. As GCN is provably less expressive than TK \cite{luan2019break}, there is a larger room to increase its expressiveness. 

Note the we use different trials for the two test label rates, the gains are generally larger when $50\%$ of the labels are available. For example, when combining with GCN, the improvements of our method are 6.29\% and 1.72\% for unlabeled Cora and Facebook test sets, but with partially-labeled test data, the improvements are 15.69\% and 2.95\% respectively. This shows the importance of modeling label dependency especially when the some test data labels are observed.

{\bf Comparison with other relational classifiers}
The two baseline non-GNN relational models, i.e. PL-EM and ICA generally perform worse than the three GNNs, with the only exception on Protein dataset. This could be because the two non-GNN models generally need a larger portion of labeled set to train the weak local classifier, whereas GNNs utilize a neural network architecture as "local classifier", which is better at representation learning by transforming and aggregating node attribute information. However, when the model is trained with a large training set (e.g. with 30\% nodes on Protein dataset), modeling the label dependency becomes crucial. At the same time, our method is still able to improve the performance of the corresponding GNNs.

For GMNN, the collective GNN model, it achieved better performance than its non-collective base model, i.e. GCN, and we can see that our model combining with GCN achieved comparable or slightly better performance than GMNN. When combing with other more powerful GNNs, our model can easily out-perform it, e.g. on Cora, Pubmed and Facebook datasets, the TK performs better than GMNN and our method adds extra gains over TK.

\begin{table*}[t!!]
\vspace{-2mm}
\centering
\resizebox{\textwidth}{!}{%
\begin{tabular}{llrrrrrrrrrr}
 & & \multicolumn{2}{c}{Cora}    & \multicolumn{2}{c}{Pubmed} & \multicolumn{2}{c}{Friendster} & \multicolumn{2}{c}{Facebook}& \multicolumn{2}{c}{Protein}\\ \cmidrule(lr){3-4}
 \cmidrule(lr){5-6} \cmidrule(lr){7-8}
 \cmidrule(lr){9-10}\cmidrule(lr){11-12}
\multicolumn{2}{r}{\bf \# train labels:} & \multicolumn{2}{c}{85 (3.21\%)}    
& \multicolumn{2}{c}{300 (1.52\%)} 
& \multicolumn{2}{c}{641 (1.47\%)} 
& \multicolumn{2}{c}{80 (1.76\%)}
& \multicolumn{2}{c}{7607 (30\%)} \\ 
\cmidrule(lr){3-4}
 \cmidrule(lr){5-6} \cmidrule(lr){7-8}
 \cmidrule(lr){9-10}\cmidrule(lr){11-12}
\multicolumn{2}{r}{\bf \% labels in $G^\text{(te)}$:}  &  \multicolumn{1}{c}{0\%} & \multicolumn{1}{c}{50\%} & \multicolumn{1}{c}{0\%} & \multicolumn{1}{c}{50\%} & \multicolumn{1}{c}{0\%}  &\multicolumn{1}{c}{50\%} & \multicolumn{1}{c}{0\%}           & \multicolumn{1}{c}{50\%}  & \multicolumn{1}{c}{0\%}  & \multicolumn{1}{c}{50\%}   \\
\cmidrule(lr){3-4}
 \cmidrule(lr){5-6} \cmidrule(lr){7-8}
 \cmidrule(lr){9-10} \cmidrule(lr){11-12}
Random & & 14.28 (0.00) & 14.28 (0.00)& 33.33 (0.00)& 33.33 (0.00)& 20.00 (0.00)& 20.00 (0.00) & 50.00 (0.00)& 50.00 (0.00)& 50.00 (0.00) & 50.00 (0.00)\\ 
\cmidrule{1-2}
\multirow{2}{*}{GCN$^\text{\cite{kipf2016semi}}$} &  -    &  45.90 (3.26)  & 36.38 (1.35) & 52.68 (2.36) & 54.11 (4.86) & 29.34 (0.55) & 28.44 (0.56)  &65.85 (1.01) & 63.13 (2.12) & 75.86 (1.11) & 77.54 (1.09)\\
&\plusours  &  \cellcolor{LightCyan} \textbf{+6.29 (1.49)}  &  \cellcolor{LightCyan}  \textbf{+15.69 (3.20)}  &  \cellcolor{LightCyan} +4.48 (2.33) &  \cellcolor{LightCyan} \textbf{+5.62 (1.17)} &  \cellcolor{LightCyan} \textbf{+0.81 (0.10)}$^*$ &  \cellcolor{LightCyan} {\bf +0.90 (0.32)} &  \cellcolor{LightCyan} \textbf{+1.72 (0.48)} & \cellcolor{LightCyan}  \textbf{+2.95 (0.84)}  &  \cellcolor{LightCyan} \textbf{+1.22 (0.51)} &  \cellcolor{LightCyan} \textbf{+0.75 (0.33)}\\
\cmidrule{1-2}
\multirow{2}{*}{GS$^\text{\cite{hamilton2017inductive}}$} & -      &  50.69 (1.50) & 48.42 (2.82) &59.34 (3.47) & 58.52 (5.42) & 28.10 (0.59) & 28.10 (0.48) & 64.56 (0.92) & 62.99 (0.88) & 73.85 (1.12) &73.01 (2.28)\\
 & \plusours  & \cellcolor{LightCyan}  \textbf{+2.35 (0.56)} &  \cellcolor{LightCyan}  \textbf{+4.52 (0.84)}  & \cellcolor{LightCyan} \textbf{+1.48 (0.41)} &  \cellcolor{LightCyan} \textbf{+2.42 (0.27)} &  \cellcolor{LightCyan} +0.31 (0.15)  &  \cellcolor{LightCyan}  +0.73 (0.23)&  \cellcolor{LightCyan} \textbf{+2.38 (0.77)} &  \cellcolor{LightCyan} \textbf{+2.05 (0.04)} &  \cellcolor{LightCyan} \textbf{+0.84 (0.12)} &  \cellcolor{LightCyan}  +1.47 (0.63)\\ 
 \cmidrule{1-2}
\multirow{2}{*}{TK$^\text{\cite{luan2019break}}$} & - &  63.74 (2.61) & 55.68 (2.08) & 61.13 (5.03) &  63.05 (5.15) & 28.89 (0.10) & 29.30 (0.15) & 67.63 (1.03) & 65.80 (1.16) & 73.65 (1.69) & 78.94 (1.50)\\
 & \plusours &  \cellcolor{LightCyan}  \textbf{+0.96 (0.30)}$^*$ & \cellcolor{LightCyan}  {\bf +7.18 (1.88)}$^*$ &    \cellcolor{LightCyan} \textbf{+1.00 (0.21)}$^*$  &  \cellcolor{LightCyan} {\bf +1.91 (0.75)}$^*$  &  \cellcolor{LightCyan} \textbf{+0.55 (0.17)}  & \cellcolor{LightCyan}  \textbf{+0.45 (0.08)}$^*$ & \cellcolor{LightCyan}  \textbf{+0.63 (0.26)}$^*$ &  \cellcolor{LightCyan} {\bf +2.37 (0.80)}$^*$ &  \cellcolor{LightCyan}  \textbf{+1.31(0.27)}$^*$ & \cellcolor{LightCyan} \textbf{+1.36(0.94)}\\
\cmidrule{1-2}
PL-EM$^\text{\cite{pfeiffer2015overcoming}}$& - & 20.70 (0.05) & 20.35 (0.05) & 38.05 (4.85)& 31.70 (4.78) & 23.26 (0.01) & 26.30 (0.25) & 56.17 (7.42) & 54.56 (6.17) &78.46 (1.45) & 77.95 (1.56)\\
ICA$^\text{\cite{lu2003link}}$& - &  26.20 (0.51)& 31.17 (3.66) & 44.40 (1.92)& 33.38 (4.69) & 25.14 (0.03) & 25.08 (0.17) & 47.93 (6.04) & 59.39 (3.69) & 84.88 (3.35)$^*$ & 84.39 (4.08)$^*$\\
GMNN$^\text{\cite{qu2019gmnn}}$& -&  49.05 (1.86)& 49.36 (2.22) & 58.03 (3.26)& 62.16 (4.40) & 22.20 (0.07) & 28.53 (0.64) & 65.82 (1.30) & 63.45 (2.15) & 76.75 (0.74) & 75.96 (0.76)\\
 \bottomrule
\end{tabular}
}
\caption{Node classification accuracy with unlabeled and partially-labeled test data. Numbers in bold represent significant improvement in a paired t-test at the $p < 0.05$ level, and numbers with $^*$ represent the best performing method in each column.}
\label{tab:results_labeled}
\vspace{-2mm}
\end{table*}

We did two ablation studies to investigate the usage of predicted labels (detailed in \Cref{sec: ablation}), which showed that (1) adding predicted labels in model input had extra value comparing to using true labels only, and (2) the gain of our framework is from using samples of the predicted labels rather than random one-hot vectors.

{\bf Runtime analysis.} \ourmodel computes $K$ embeddings at each stochastic gradient step, therefore 
overall, per-gradient step, \ourmodel is $K$ slower than its component WL-GNN. Overall, after $T$ iterations of Steps 1-3, \ourmodel total runtime increases by $T \times K$ over the original runtime of its component WL-GNN.
%
Our experiments are conducted on a single Nvidia GTX 1080Ti with $11$ GB of shared memory. \ourmodel built on GCN on the largest dataset (i.e. Friendster with 43K nodes) takes $66.31$ minutes for training and inference (the longest time), with a total of $T=10$ iterations, while the corresponding GCN takes only $1.04$ minutes for the same operations. In the same dataset, \ourmodel built on TK takes $132.33$ minutes for training and inference, while the corresponding TK takes $1.93$ minutes.
We give a detailed account of running times on smaller graphs in \Cref{sec: append_runtime}.
We note that we spent nearly no time engineering \ourmodel for speed or for improving our results. 
Our interest in this paper lies entirely on the gains of a direct application of \ourmodel. 
We fully expect that further engineering advances can significantly reduce the performance penalty and increase accuracy gains.
For instance, parallelism can significantly reduce the time to collect $K$ samples in \ourmodel.

%% file: 6_relatedwork.tex
\vspace{-2mm}
\section{Related work}
\vspace{-2mm}

{\bf On collective learning and neural networks.}
There has been work on applying deep learning to collective classification. For example, \citet{moore2017deep} proposed to use LSTM-based RNNs for classification tasks on graphs. They transform each node and its set of neighbors into an unordered sequence and use an RNN to predict the class label as the output of that sequence. \citet{pham2017column} designed a deep learning model for collective classification in multi-relational domains, which learns local and relational features simultaneously to encodes multi-relations between any two instances. 

The closest work to ours is \citet{fan2019recurrent}, which proposed a recurrent collective classification (RCC) framework, a variant of ICA \cite{lu2003link} including dynamic relational features encoding label information. Unlike our framework, this method does not sample labels $\hat{\mY}$, opting for an end-to-end training procedure.
\citet{vijayan2018hopf} opts for a similar no-sample RCC end-to-end training method as \cite{fan2019recurrent}, now combining a differentiable graph kernel with an iterative stage.
Graph Markov Neural Network (GMNN) \cite{qu2019gmnn} is another promising approach that applies statistical relational learning to GNNs.
GMNNs model the joint label distribution with a conditional random field trained with the variational EM algorithm. 
GMNNs are trained by alternating between an E-step and an M-step, and two WL-GCNs are trained for the two steps respectively.
These studies represent different ideas for bringing the power of collective classification to neural networks.
Unfortunately, \Cref{coro:gnn} shows that, without sampling $\hat{\mY}$, the above methods are still WL-GNNs, and hence, their use of collective classification fails to deliver any increase in expressiveness beyond an optimal WL-GNN (e.g.,~\citet{xu2018powerful}).

In parallel to our work, \citet{jia2020outcome} considers regression tasks by modeling the joint GNN residual of a target set ($y - \hat{y}$) as a multivariate Gaussian, defining the loss function as the marginal likelihood only over labeled nodes $\hat{y}_L$.  In contrast, by using the more general foundation of collective classification, our framework can seamlessly model both classification and regression tasks, and include model predictions over the entire graph $\hat{\mY}$ as \ourmodel{}'s input, thus affecting both the model prediction and the GNN training in inductive node classification tasks. 


{\bf On self-supervised learning and semi-supervised learning.}
Self-supervised learning is closely related to semi-supervised learning.
In fact, self-supervised learning can be seen as a self-imposed semi-supervised learning task, where part of the input is masked (or transformed) and must be predicted back by the model~\citep{doersch2015unsupervised,noroozi2016unsupervised,lee2017unsupervised,misra2016shuffle}.
Recently, self-supervised learning has been broadly applied to achieve state-of-the-art accuracy in computer vision~\citep{henaff2019data,gidaris2019boosting} and natural language processing~\cite{devlin2018bert} supervised learning tasks. 
The use of self-supervised learning in graph representation learning is intimately related to the use of pseudolikelihood to approximate true likelihood functions.

\input{furtherrelatedwork}


%% file: furtherrelatedwork.tex
{\bf Collective classification for semi-supervised learning tasks.}
Conventional relational machine learning (RML) developed methods to learn joint models from labeled graphs \citep{lu2003link,neville2000iterative}, and applied the learned classifier to {\em jointly} infer the labels for unseen examples with {\em collective inference}.
When the goal is to learn and predict within a \textit{partially-labeled} graph, RML methods have considered semi-supervised formulations \citep{koller2007introduction, xiang2008pseudolikelihood, pfeiffer2015overcoming} to model the joint probability distribution:
\[ 
P (\mY_U|\mY_L,\mX, \mA).
\]
In this case RML methods use both {\em collective learning} and {\em collective inference} procedures for semi-supervised learning. 

RML methods typically consider {\em a Markov assumption} to simplify the above expressions ---every node $v_i \in V$ is considered conditionally independent of the rest of the network given its Markov Blanket ($\mathcal{MB}(v_i)$). For undirected graphs, this is often simply the set of the immediate neighbors of $v_i$. Given the Markov blanket assumption, RML methods typically use a local conditional model (e.g., relational Naive Bayes \citep{neville2003simple}, relational logistic regression \citep{popescul2002towards}) to learn and infer labels within the network.

Semi-supervised RML methods utilizes the {\em unlabeled data} to make better predictions within the network. Given the estimated values of unlabeled examples, i.e.$P(\mY_U^\text{(tr)}|\mY_L^\text{(tr)}, \mX^\text{(tr)}, \mA^\text{(tr)})$, the local model parameters can be learned by maximizing the {\em pseudolikelihood} of the labeled part:

\begin{equation} \label{eq:RML1}
O = \sum_{v \in V^\text{(tr)}}{\log P(y_v^\text{(tr)} | \mY^\text{(tr)}_{\mathcal{MB}(v)}, \mX^\text{(tr)}, \mA^\text{(tr)})}.
\end{equation}

The {\em key} difference between \Cref{eq:RML1} and the GNNs objective in \Cref{equ:gnn_obj}: the RML model is always conditioned on the labels (either true labels $\mY^\text{(tr)}_L$ or estimated labels $\hat{\mY}^\text{(tr)}_U$) even when there are no observed labels in the test data, i.e., even when $\mY_L^\text{(te)} = \emptyset$.

The most common form of semi-supervised RML utilizes expectation maximization (EM)~\citep{xiang2008pseudolikelihood, pfeiffer2015overcoming}, which iteratively relearn the parameters given the expected values of the unlabeled examples. For instance, the PL-EM algorithm \cite{pfeiffer2015overcoming} optimizes the pseudolikelihood over the entire graph:
\begin{align*}
&\textbf{E-Step: } \text{evaluate}~ P(\mY_U^\text{(tr)}|\mY_L^\text{(tr)}, \mX^\text{(tr)}, \mA^\text{(tr)}, \Theta_{t-1})\\
&\textbf{M-Step: } \text{learn}~ \Theta_{t}:\\
& \Theta_{t} = \argmax_{\Theta}\sum_{\mY_U^\text{(tr)}\in \Gamma_U}{P(\mY_U^\text{(tr)}|\mY_L^\text{(tr)}, \mX^\text{(tr)}, \mA^\text{(tr)}, \Theta}_{t-1}) 
  \sum_{v \in V^\text{(tr)}}{\log P(y_v^\text{(tr)}|\mY_{\mathcal{MB}(v)}^\text{(tr)}, \mX^\text{(tr)}, \mA^\text{(tr)}, \Theta)}.
\end{align*}

In comparison to semi-supervised RML, our proposed framework \ourmodel performs collective learning to strengthen GNN, which is a more powerful and flexible "local" classifier compared to the typical weak local classifier used in RML methods (e.g. relational Naive Bayes). For parameter learning (M-step), \ourmodel also optimizes pseudolikelihood, but incorporates Monte Carlo sampling from label prediction $\hat{\mY}^{(t-1)}$ instead of directly using the predicted probability of unlabeled nodes $P(\mY_U^\text{(tr)}|\mY_L^\text{(tr)}, \mX^\text{(tr)}, \mA^\text{(tr)}, \Theta_{t-1}).$

{\em Collective inference.} When the model is applied to make predictions on unlabeled nodes (in the E-step), {\em joint (i.e., collective) inference} methods such as variational inference or Gibbs sampling must be applied in order to use the conditionals from \Cref{eq:RML1}. This combines the local conditional probabilities with global inference to estimate the joint distribution over the unlabeled vertices, e.g.,:
\[P(\mY_U|\mY_L, \mX) \approx Q(\mY_U) = \Pi_{v_i \in V_U}{Q_i(y)},\]
where each component $Q_i(y)$ is iteratively updated.

Alternatively, a Gibbs sampler iteratively draws a label from the corresponding conditional distributions of the unlabeled vertices:
\[ \hat{y}_v \sim P(y_v | \hat{Y}_{\mathcal{MB}(v)}, \mY_L,  \mX, A), \; \forall v \in V. \]
In Gibbs sampling, it is the sampling of labels that allow us to sample from the joint distribution, which is what enriches the simple models often used in collective inference.


%% file: 7_conclusion.tex
\vspace{-2mm}
\section{Conclusion}
\vspace{-2mm}
In this work, we answer the question ``can collective learning and sampling techniques improve the expressiveness of state-of-the-art GNNs in inductive node classification tasks?'' 
We first show that with the most expressive GNNs there is no need to do collective learning; however, since we do not have the most expressive models, we present the collective learning framework (exemplified in \ourmodel), that can be combined with any existing GNNs to provably improve WL-GNN expressiveness. We considered the inductive node classification tasks across graphs, and showed by extensive empirical study that our collective learning significantly boosts GNNs performance on all the tasks.

%% file: 8_impact.tex
\section{Potential biases}
This work presents an algorithm for node classification that works on arbitrary graphs, including real-world social networks, citation networks, etc..
Any classification algorithm that learns from data runs the risk of producing biased predictions reflective of the training data --- our work, which learns an inductive node classifier from training graph examples, is no exception. However, our main focus in this paper is on introducing a general algorithmic framework that can increase the expressivity of existing graph representation learning methods, rather than particular real-world applications.

%% file: 9_appendix.tex
\clearpage
\onecolumn
\appendix

\section*{Supplementary Material of collective learning GNN}

\section{Proof of \Cref{thm:col}}

We restate the theorem for completeness.
\thmCCun*
\begin{proof}
Let $\hat{Y}(v) = \varphi (\Gamma^\star(v,G))$ be a classifier function that takes the most expressive representation $\Gamma^\star(v,G)$ of node $v$ as input and outputs a predicted class label for $v$. 

Let $\hat{\mY}^{(t)}$ be the set of predicted labels at iteration $t$ of collective classification and let $Y_v$ be the true label of node $v \in V$. 
Then either (1) $Y_v \indep_{\Gamma^\star(v,G),\varphi} \hat{\mY}^{(t)}$, or (2) $Y_v \not\!\indep_{\Gamma^\star(v,G),\varphi} \hat{\mY}^{(t)}$. 

Case (1): Given the classifier $\varphi$ and the most expressive representation $\Gamma^\star(v,G)$, the true label of $v$ is independent of the labels predicted with collective classification. In this case, the predicted labels of $v$'s neighbors offer no additional information and, thus, collective classification is unnecessary.

Case (2): In this case, the true label of $v$ is not independent of the predicted labels. By Theorem 1 of \citet{srinivasan2019equivalence}, we know that for any random variable $H_v$ attached to node $v \in V$, it must be that $\exists \varphi'$ a measurable function independent of $G$ s.t.\ 
\[
H_v \stackrel{\text{a.s.}}{=} \varphi' (\Gamma^\star(v,G),\epsilon_v),\]
where $\epsilon_v$ is an noise source exogenous to $G$ (pure noise), and a.s.\ implies almost sure equality. Defining $H_v := Y_v$,  
\[ \varphi' (\Gamma^\star(v,G),\epsilon_v) \not\!\indep_{\Gamma^\star(v,G),\varphi} \hat{\mY}^{(t)}, \]
which means $\hat{\mY}^{(t)}$ must either be dependent on $\epsilon_v$ or contain domain knowledge information about the function $\varphi'$ that is not in $\varphi$. Since $\hat{\mY}^{(t)}$ is a vector of random variables fully determined by $G$ and $\varphi$, it cannot depend on an exogenous variable $\epsilon_v$, Thus, the predictions must contain domain knowledge of $\varphi'$. 
Hence, we can directly incorporate this domain knowledge into another classifier $\varphi^\dagger$ s.t. $Y_v \indep_{\Gamma^\star(v,G),\varphi^\dagger} \hat{\mY}^{(t)}$, for instance $\varphi^\dagger$ is a function of $\varphi'$. In this case, $\varphi^\dagger$ will predict the label of $v$ with equal or higher accuracy than collective classification based on predicted labels $\hat{\mY}$, which finishes our proof.
\vspace{5mm}


\end{proof}

\section{Proof of \Cref{thm:cl-power}}

\begin{figure}[h!]
        \centering
        \includegraphics[width=15cm]{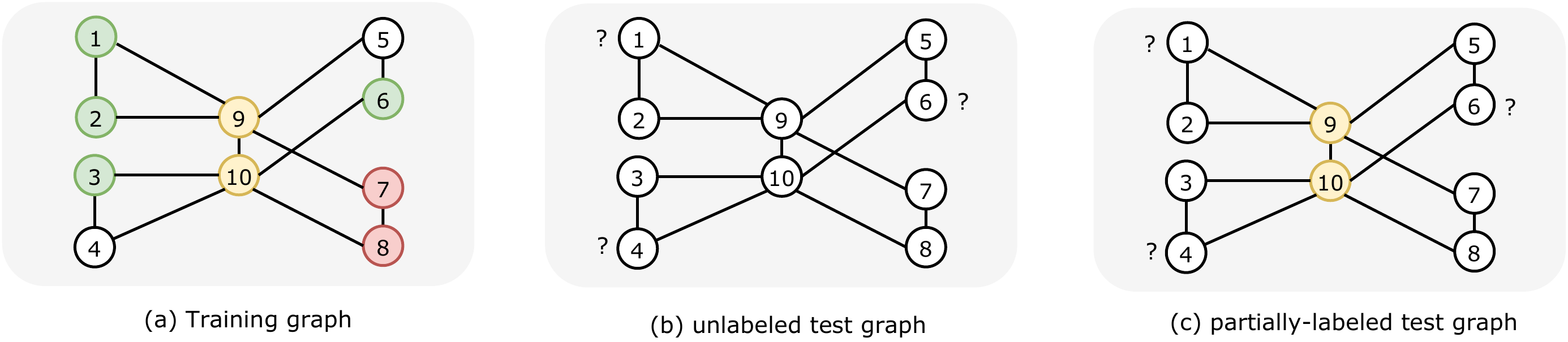}
        \caption{Training/testing graphs. Colors represent available node labels, and testing nodes are marked with question marks. WL-GNN cannot differentiate between the red and green nodes. }
        \label{fig:tri}
    \end{figure}
    
\thmclpower*
\begin{proof} 
As defined, $\text{WLGNN}$ is a most-expressive WL-GNN.
We need to prove $\cG(\text{WL-GNN}) \subsetneq \cG(\text{\ourmodel}).$
We will do that by first showing $\cG(\text{\ourmodel}) = \cG(\text{WL-GNN}) \cup S'$ and then showing that $\exists G \in \cG(\text{\ourmodel}) \text{ s.t. } G \not\in \cG(\text{WLGNN})$.

{\em $\cG(\text{\ourmodel}) = \cG(\text{WL-GNN}) \cup S'$}: First, we need to show that for any mask $M \in \cM$, $\exists S \subseteq \cG(\text{\ourmodel})$ such that $S = \cG(\text{WL-GNN})$.
This is clearly true since, for labeled tests, in \Cref{eq:Z_labeled} we can always construct a $\text{WLGNN}^0$ for a \ourmodel 
\begin{equation} \label{eq:Z0}
\begin{split}
\mZ_v^{(t)}(\text{WLGNN}^0)= \expectation_{\hat{\mY}^{(t-1)}} \big[ &\text{WLGNN}^0(\mX^\text{(tr)}, \mY^{\text{(tr)}}_{L}\odot M, 
 \hat{\mY}^{(t-1)}\odot \overline{M}, \mA^\text{(tr)} ; \Theta)_v \big],
\end{split}
\end{equation}
that ignores the $\hat{\mY}$ inputs.
Similarly, for unlabeled tests,  in \Cref{eq:Z_labeled}  we can always construct a $\text{WLGNN}^1$ for a \ourmodel 
\begin{equation*}
        \mZ_v^{(t)}(\text{WLGNN}^1) = \expectation_{\hat{\mY}^{(t-1)}} \left[ \text{WLGNN}^1(\mX^\text{(tr)}, \hat{\mY}^{(t-1)}, \mA^\text{(tr)} ; \Theta)_v \right],
\end{equation*}
that ignores the $\hat{\mY}^{(t-1)}$ inputs.

{\em $\exists G \in \cG(\text{\ourmodel}) \text{ s.t. } G \not\in \cG(\text{WL-GNN})$}:
Let $G$ be the graph in \Cref{fig:tri}.
We will first consider the case where the test data has partial labels. The case without labels follows directly from it.
Using the graph $G$ in \Cref{fig:tri}(a) (training) and \Cref{fig:tri}(c) (partially-labeled testing) we show that a WL-GNN is unable,  in test, to correctly give representations to the left-most nodes $\{1,2,3,4\}$ that are distinct from the right-most nodes $\{5,6,7,8\}$ (the same happens for the unlabeled test graph in \Cref{fig:tri}(b)).
We then show that the representation $\mZ^{(t)}_v$ of \Cref{eq:Z0} is able to distinguish these two sets of nodes.

{\em WL-CNN is not powerful enough  to give distinct representations to nodes $\{1,\ldots,8\}$ in \Cref{fig:tri}(c)}:
Consider giving an arbitrary feature value (say, the ``color white'') to all uncolored nodes $\{1,\ldots,8\}$ in \Cref{fig:tri}(c).
We will start showing that the 1-WL test is unable to give different colors to the nodes $\{1,\ldots,8\}$ in this graph.
Since WL-GNNs are no more expressive than the 1-WL test~\citep{xu2018powerful,morris2019weisfeiler}, showing that the above is a stable coloring for nodes $\{1,\ldots,8\}$ in the 1-WL test, proves the first part of our result.
A stable 1-WL coloring is defined as a coloring scheme on the graph that has a 1-to-1 correspondence with the colors in the previous step of the 1-WL algorithm.
Since the input to the hash function of the 1-WL test is the same for all of nodes $v \in \{1,\ldots,8\}$: The node itself has color white while the color set of the neighbors is the set $\{\text{white},\text{yellow}\}$. 
In the next 1-WL round, all the white nodes will be mapped to the same color by the hash function. The colors of node $\{9,10\}$ will be not the same as $\{1,\ldots,8\}$.
Hence, the initial coloring of all nodes $\{1,\ldots,8\}$ white and $\{9,10\}$ yellow is a stable coloring for 1-WL.
Consequently, WL-GNN will give the same representation to all nodes in $\{1,\ldots,8\}$.

    
    
\ourmodel\ {\em gives the same representations within the sets $\{1,\ldots,4\}$ and $\{5,\ldots,8\}$}: 
At iteration $t \geq 0$ of \ourmodel,
we start with the base of the recursion $\mY^{(t-2)} = {\bf 0}$.
Now consider a given mask $\mM^{(t)}  \in \cM$.
 Note that to sample $\hat{\mY}^{(t-1)}_v$ for $v \in \{1,\ldots,8\}$ we apply $\hat{\mY}^{(t-2)} = {\bf 0}$ into  \Cref{eq:Z_labeled} to obtain $\mZ_v^{(t-1)}$, and then apply $\mZ_v^{(t-1)}$ into \Cref{eq:yhat}, defining $\mX_{\mY_L^\text{(tr)},\hat{\mY}^{(t-1)},\mM^{(t)}}^\text{(tr)} = [\mX^\text{(tr)},\mY^{\text{(tr)}}_{L}\odot \mM^{(t)} , {\bf 0}]$, which will give us $\hat{\mY}^{(t-1)}_v \sim P(\mY^{(t-1)}_v \, | \, \text{WLGNN}([\mX^\text{(tr)},\mY^{\text{(tr)}}_{L}\odot \mM^{(t)} , {\bf 0}], \mA^\text{(tr)} ; \Theta)_v)$, and any classes has a non-zero probability of being sampled since our output is a softmax.
 
Since nodes $\{1,\ldots,8\}$ all get the same representation in the above $\text{WLGNN}$, their respective sampled $\hat{\mY}^{(t-1)}_v$, $v \in \{1,\ldots,8\}$, will have the same distribution but possibly not the same values (due to sampling).
Note that the nodes $\{1,\ldots,4\}$ will get the same average in \Cref{eq:Z0} since $\hat{\mY}^{(t-1)}_v$, $v \in \{1,\ldots,4\}$, has the same distribution and the nodes are isomorphic (even given the colors on nodes 9 and 10).
Similarly, the nodes $\{5,\ldots,8\}$ will also get the same average in \Cref{eq:Z0}.

\ourmodel\ {\em  gives distinct  representations accross the sets  $\{1,\ldots,4\}$ and $\{5,\ldots,8\}$}:
Finally, we now prove that exists a WL-GNN, which we will denote $\text{WLGNN}^2$, such that $\mZ_v^{(t)}(\text{WLGNN}^2) \neq \mZ_u^{(t)}(\text{WLGNN}^2)$ for $v \in \{1,\ldots,4\}$ and $u\in \{5,\ldots,8\}$.
We will show that there is a joint sample of $\hat{\mY}_1^{(t-1)},\ldots,\hat{\mY}_8^{(t-1)}$ where there is no symmetry between the representations of nodes in $\{1,\ldots,4\}$ and $\{5,\ldots,8\}$.
Since each layer of WLGNN$^2$ can have different parameters, we can easily encode differences in the number of hops it takes to reach a certain color. 
Moreover, at any WLGNN$^2$ layer, the representation of a node can perfectly encode its own last-layer representation and the last-layer representation of its neighbors through a most-expressive multiset representation function~\cite{xu2018powerful}.

It is enough for us to show that for a sampled $\hat{\mY}^{(t-1)}$ the sets of nodes $\{1,\ldots,4\}$ and $\{5,\ldots,8\}$ can get distinct unique representations under WLGNN$^2$.
By unique, we mean, $\{1,\ldots,4\}$ can get representations in WLGNN$^2$ that cannot be obtained by the nodes in $\{5,\ldots,8\}$.
This representation uniqueness makes sure the averages in \Cref{eq:Z_labeled} are different.
Without loss of generality we will consider giving a special sampled label to only one node $i \in \{1,5\}$ in one of the sets.
The sampled labels
$\hat{\mY}_i^{(t-1)} = \text{green}$, while all other nodes $\{1,\ldots,8\}\backslash \{i\}$ get red, will happen with non-zero probability, hence, they must be part of the expectation in \Cref{eq:Z_labeled}.
Note that node $2$ (for $i=1$) and $6$ (for $i=5$) will feel the effects of the green color in their neighbors differently.
That is, for $i=5$ there is a parameter choice for the layers of $\text{WLGNN}^2$ where the representation of node 6 uniquely encodes that the color green is within hops 1 (node 5) and 3 (from node 5 through nodes 9 and 10) of node 6 (if 6 treats its own representation differently from its neighbors). 
For $i=1$, node 2's representation will encode that green is observed hops 1 (node 1) and 2 (from node 1 through node 9) (similarly, 2 treats its own representation differently from its neighbors).
Hence, these representations can be made unique by $\text{WLGNN}^2$, i.e., no other $\hat{\mY}^{(t-1)}$ assignments will create the same patterns for nodes 2 and 6, and thus, since $\text{WLGNN}^2$ has most-expressive multiset representations, it can give a unique representation to nodes 2 and 6 for these two unique $\hat{\mY}^{(t-1)}$ configurations.
These unique representations are enough to ensure $\mZ_2^{(t)}(\text{WLGNN}^2) \neq \mZ_6^{(t)}(\text{WLGNN}^2)$ for any $t \geq 1$, which concludes our proof.

\end{proof}

\section{Proof of \Cref{prop:expand-power}}

\begin{figure}[h!]
    \hspace{-0.8cm}
    
    \subfloat[Training graph]{\includegraphics[width=5cm]{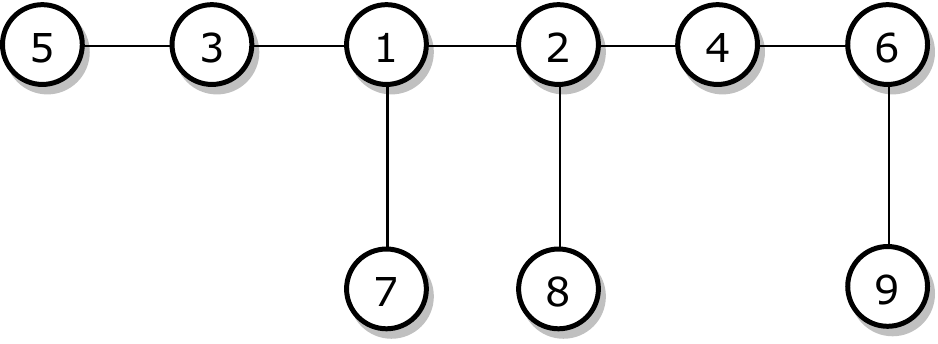}}
    \hspace{0.8cm}
    \subfloat[$2^{nc}$-order neighborhood for label prdiction]{\includegraphics[width=7cm]{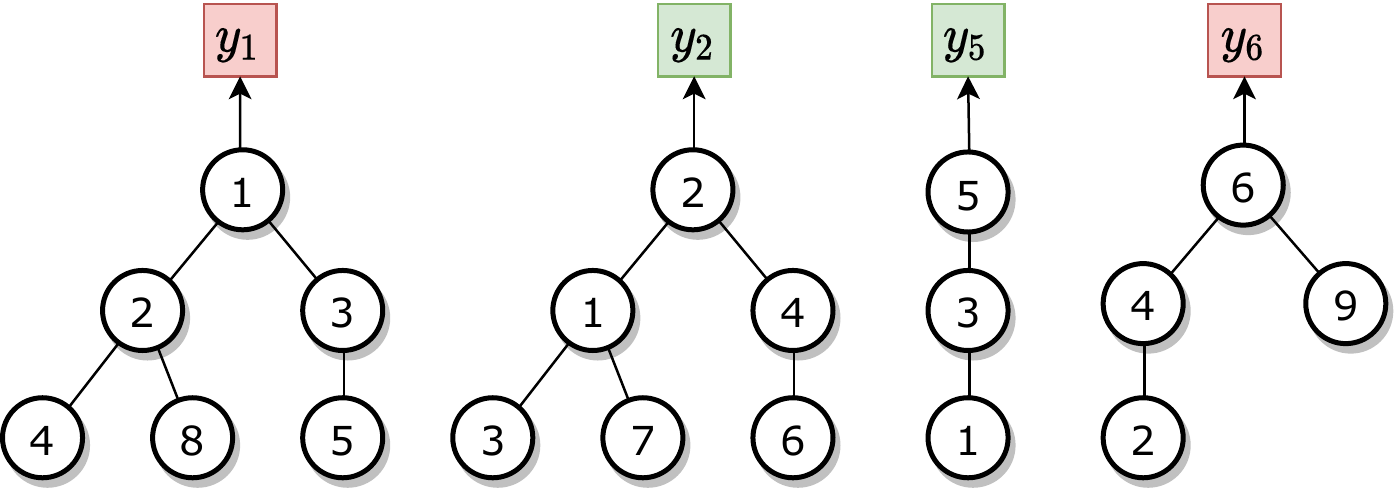}}
    \caption{WL-GNN using 2nd-order neighborhood cannot differentiate node 1 and 2, but \ourmodel built on this WLGNN can break the local isomorphism.}
    \label{fig:line_graph}
\end{figure}

\propexpandpower*
\begin{proof}
Let $G$ be the graph in \Cref{fig:line_graph}(a) with no node features, and let WLGNN be of order $2$, meaning it will generate node embeddings based on 2nd-order neighborhoods (shown in (b)). Since node $1$ and $2$ have the same 2nd-order neighborhood structure, WLGNN will generate identical node representation for them, which gives random label predictions. Meanwhile, as nodes $5$ and $6$ have distinct 2nd-order neighborhood structures, WLGNN generates different node representations for them, which enables the model to learn from the labels $y_5$ and $y_6$. We can assume the predicted label probability $P(\hat{\mY}^{(0)}_5= \text{green}) = 0.99$ and $P(\hat{Y}^{(0)}_6 = \text{red}) = 0.98$. For \ourmodel, at iteration $t = 1$, we sample $\hat{\mY}^{(0)}$ from the WLGNN output $P(\hat{\mY}^{(0)})$ and use the samples as input. In the worst case, nodes $3, 4$ and $7, 8$ get the same distribution and sampled labels (i.e. $\hat{\mY}_3^{(0)}$ = $\hat{\mY}_4^{(0)}$, $\hat{\mY}_7^{(0)}$ = $\hat{\mY}_8^{(0)}$). Since the distribution of $\hat{\mY}_5^{(0)}$ and $\hat{\mY}_6^{(0)}$ are different, the samples of $\hat{\mY}_5^{(0)}$ and $\hat{\mY}_6^{(0)}$ are different, which breaks the tie between the 2nd-order neighborhood of nodes $1$ and $2$. Therefore, \ourmodel will produce different node representation starting from iteration $t = 1$ for nodes $1$ and $2$, which enables the model to learn from the training label $y_1$ and $y_2$, and thus gives more accurate predictions. 
\end{proof}

The advantage of collective inference is more clear when it is used to strengthen less-expressive local classifiers, e.g. logistic regression
Although GNN are much powerful than these local classifiers by aggregating high(er)-order graph information, collective learning can still help if GNN fail to make use of ``global'' information in graphs (or equivalently, if the order of GNN is small than graph diameter).
Previous work \cite{jensen2004collective} investigating the power of collective inference also showed that methods for collective inference benefit from \textit{a clever factoring of the space of dependencies}, by arguing that \textit{these (collective inference) methods benefit from information propagated from outside of their local neighborhood. Predictions about the class label on other objects essentially ``bundle information'' about the graph beyond the immediate neighborhood.} 

\section{Proof of \Cref{prop:opt}}
\propopt*
\begin{proof}
In our optimization, we only need to sample two variables 
$\hat{\mY}^{(t-1)}$ and $\mM^{(t)}$.
We obtain unbiased bounded-variance estimates of the derivative of the loss function if we sample $\mM^{(t)} \sim \text{Uniform}(\cM)$ (and exact values when $\mM^{(t)} = {\bf 0}$).
We can now compound that with unbiased bounded-variance estimates of the derivative if we estimate the expectation in \Cref{eq:Z_labeled} $\{\mZ^{(t-1)}_v\}_{v \in V}$ by i.i.d.\ sampling $\hat{\mY}^{(t-1)}$.
The loss in \Cref{eq:mcgnn_obj_mask} is convex on $\mZ_v^{(t)}$ since the negative log-likelihood of the multi-class logistic regression is convex on ${\bf W}$, which means it is also convex on $\mZ_v^{(t)}$ as the loss is defined on the affine transformation ${\bf W}\mZ_v^{(t)}$. 
The expectation of the loss always exist, since we assume $\nabla_{\Theta} ({\bf W}\mZ_v^{(t)}(\mM^{(t)};\Theta))_{y^\text{(tr)}_v}$ is bounded for all $\forall v \in V_L^\text{(tr)}$.
Hence, as the loss is convex w.r.t.\ $\mZ_v^{(t)}$, the expection w.r.t.\ $\mZ_v^{(t)}$ exists, and we obtain an unbiased estimate of $\mZ_v^{(t)}$, we can apply Jensen's inequality to show that the resulting Robbins-Monro stochastic optimization optimizes an upper bound of the loss in \Cref{eq:mcgnn_obj_mask}.
\end{proof}

\section{Additional information on datasets and experiment setup}\label{appx:datasets}

{\em Datasets.}
We use five datasets for evaluation, with the dataset statistics shown in \Cref{tab:data_stat}. 
\begin{table}[h!]
    \centering
    \caption{Dataset statistics}
    \label{tab:data_stat}
    \begin{tabular}{c|c|c|c|c} \toprule
    Dataset & \# Nodes  & \# Attributes & \# Classes & \# Test \\ \hline
    Cora    & 2708 & 1433 &  7 & 1000\\
    Pubmed  & 19717 & 500 & 3  & 1000\\
    Friendster & 43880 & 644 &5 & 6251 \\
    Facebook & 4556 & 3 & 2 & 1000 \\ 
    Protein & 12679 & 29 & 2 & 2376\\ \bottomrule
    \end{tabular}
\vspace{-2mm}
\end{table}

\begin{itemize}[leftmargin=*]
\vspace{-2mm}
    \item {\bf Cora} and {\bf Pubmed} are benchmark datasets for node classification tasks from \cite{sen2008collective}. They are citation networks with nodes representing publications and edges representing citation relation. Node attributes are bag-of-word features of each document, and the predicted label is the corresponding research field. 
\vspace{-1mm}
    \item {\bf Facebook} \cite{yang2017stochastic} is social network of Facebook users from Purdue university, where nodes represent users and edges represent friendship. The features of the nodes are: religious views, gender and whether the user’s hometown is in Indiana. The predicted labels is political view. 
\vspace{-1mm}
    \item {\bf Friendster} \cite{teixeira2019graph} is social network. Nodes represent users and edges represent friendship. The node attributes include numerical features (e.g number of photos posted, etc) and categorical features (e.g. gender, college, music interests, etc), encoded as binary one-hot features. The node labels represent one of the five age groups: 0-24, 25-30, 36-40, 46-50 and over 50. This  version of the graph contain 40K nodes, 25K of which are labeled.
\vspace{-1mm}
    \item {\bf Protein} is a collection of protein graphs  from \cite{borgwardt2005protein}. Each node is labeled with a functional role of the protein, and has a 29 dimensional feature vector. We use 85 graphs with an average size of 150 nodes. 
\vspace{-1mm}
\end{itemize}

{\em Data splits.}
To conduct inductive learning tasks, we have to properly split the graphs into labeled and unlabeled parts. For datasets containing only one graph (Cora, Pubmed, Facebook and Friendster), we randomly sample a connected component to be $ V_L^\text{(tr)}$, and then sample a test set ($V_T$) from the remainder nodes ($V_U$). To make partially-labeled test data available, we sample another connected component as $V_L^\text{(te)}$ with the same size as $ V_L^\text{(tr)}$. The nodes are sampled to guarantee that there is no overlapping between any two sets of $ V_L^\text{(tr)}$, $ V_L^\text{(te)}$ and $V_T$. Here $G^{\text{(tr)}}$ and $G^{\text{(te)}}$ have the same graph structure but with different labeled nodes.

For the protein dataset, as we have 85 disjoint graphs, we randomly choose 51 (60\%) graphs for training, 17 (20\%) graphs for validation and the remaining 17 (20\%) graphs for testing. To simulate semi-supervised learning settings, we mask out 50\% of true labels on the training graphs. For the tasks with partially-labeled test data, we randomly select 50\% of the nodes in the test graphs as labeled nodes, and test on the remaining 50\% nodes. We run five trials for all the experiments, and in each trial we randomly split the nodes/graphs as described. 

As seen in \Cref{sec: exp}, to approximate an inductive learning setting, we use a different train/test data split procedure (i.e. connected training set) on Cora and Pubmed networks from the public version (i.e. random training set) used in most of the existing GNN models \citep{kipf2016semi, luan2019break}. This is illustrated in \Cref{fig:data_split}, where the random training set of the traditional GNN evaluation methods (in e.g., \citep{kipf2016semi, luan2019break}) is shown on the left, contrasted with our harder task of connected training set shown on the right. 
This difference in task is the reason why the model performance reported in our paper is not directly comparable with the reported results in previous GNN papers, even though we used the same implementations and hyperparameter search procedures.

\begin{figure}[h!]
    \centering
    \includegraphics[width=14cm]{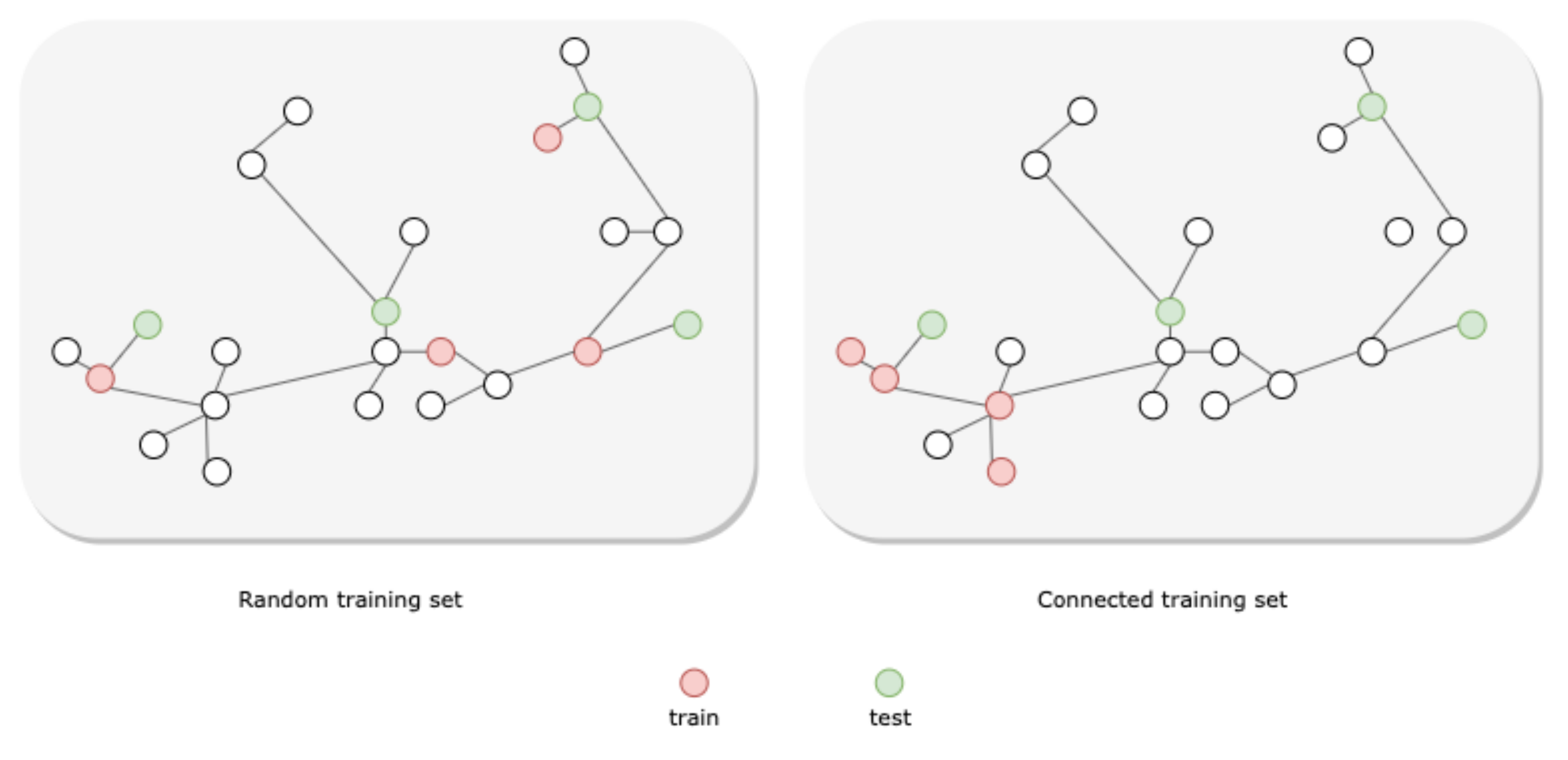}
    \caption{The different data splits between traditional GNN train/test split evaluation (left) and our---more realistic--- connected train/random test split evaluation (right)}
    \label{fig:data_split}
\end{figure}

{\em Hyperparameter setting.}
For hyperparameter tuning, we searched the initial learning rate within \{0.005, 0.01, 0.05\} with weight decay of $0.0005$. Dropout is applied to all the layers with $p = 0.5$. Hidden units are searched within \{16, 32\} if the dataset wasn't used by the original GNN paper, or set as the same number as originally chosen in the GNN paper. The number of layers is set to $2$ for both GCN \cite{kipf2016semi} and GraphSage \cite{hamilton2017inductive} as used in their paper, and we use $10$ layers for TK \cite{xu2018powerful}. For GraphSage \cite{hamilton2017inductive}, the neighborhood sample size is set to $5$. We use the same GNN structure (i.e. layers, hidden units, neighborhood sample size) for the non-collective version and in \ourmodel for fair comparison. 

For \ourmodel, the additional hyperparameters are (1) the sample size of predicted labels $\hat{Y}$ ($K$), and (2) the number of model iterations ($T$). we set sample size $K = 8$ for friendster dataset and $K = 10$ for all other datasets. For label rate of $50\%$, the model is trained for $T = 10$ iterations, and each iteration contains $100$ epochs. Note that we sample a new binary mask for each epoch as described in \Cref{sec:framework}. For label rate of $0\%$, the model is trained for $T = 3$ iterations, and each iteration contains up to $500$ epochs which can be early stopped if the validation accuracy decreases for a specified consecutive epochs. The numbers of iterations are empirically determined as only marginal improvements are observed after $3$ iterations for unlabeled test data and $10$ iterations for partially-labeled test data. The validation accuracy is used to choose the best epoch. 

Note that the hyper-parameter tuning could be done more aggressively to further boost the performance of \ourmodel, e.g. using more layers for TK \cite{xu2018powerful}, but our main goal is to evaluate the relative improvements of \ourmodel on the corresponding non-collective GNNs.

\section{Running times for GNN models on multiple datasets} \label{sec: append_runtime}

\begin{table}[h!]
    \centering
    \resizebox{.8\textwidth}{!} {
    \begin{tabular}{ccccc}
    \multirow{3}{*}{Dataset}   & \multirow{3}{*}{GNN structure} & \multicolumn{3}{c}{Running time (minutes)} \\
    \cmidrule(lr){3-5}
    &   & \multirow{2}{*}{GNN} & \multicolumn{2}{c}{\ourmodel} \\ 
    \cmidrule(lr){4-5}
       &   &  & unlabeled $G^\text{(te)}$ & partially-labeled $G^\text{(te)}$\\ 
       \cmidrule(lr){1-1} \cmidrule(lr){2-2} \cmidrule(lr){3-3} \cmidrule(lr){4-4} \cmidrule(lr){5-5}
    \multirow{2}{*}{Cora}    & GCN    & 0.09  & 0.83 & 3.65 \\
       & TK    & 0.12  & 1.91 & 5.74 \\
       \cmidrule(lr){1-1} \cmidrule(lr){2-2}
    \multirow{2}{*}{Pubmed} & GCN & 0.49 & 5.38 & 21.87 \\
         & TK  & 0.52 & 7.82 & 51.62 \\ 
         \cmidrule(lr){1-1} \cmidrule(lr){2-2}
       \multirow{2}{*}{Friendster} & GCN & 1.04 & 17.93 & 66.31 \\
        & TK & 1.93 & 30.17  & 132.33\\ 
        \cmidrule(lr){1-1} \cmidrule(lr){2-2}
      \multirow{2}{*}{Facebook} & GCN & 0.02 & 1.44 & 5.37\\
        & TK & 0.05 & 2.41  & 7.22\\   
        \bottomrule
    \end{tabular}
    }
    \caption{The running time (in minutes) for \ourmodel and its corresponding GNNs.}
    \label{tab:running_time}
    \vspace{-4mm}
\end{table}

\Cref{tab:running_time} shows the running times for \ourmodel and the corresponding non-collective GNNs on various datasets. As mentioned in \label{appx:datasets}, for partially-labeled $G^\text{(te)}$, \ourmodel applied random masks at each epoch, and ran for $10$ iterations, whereas for unlabeled $G^\text{(te)}$, \ourmodel ran for $3$ iterations. 

\section{\ourmodel performance with varying training label rates}
To investigate the impact of the training label rates on the node classification accuracy, we repeated the experiments on Cora and Pubmed datasets with various numbers of training labels, on unlabeled test data and partially-labeled test data. \Cref{tab:results_unlabeled_learning_curve} and \Cref{tab:results_labeled_learning_curve} show the results for test labels rates of $0\%$ and $50\%$ respectively. We can see that in general \ourmodel framework achieved a larger improvement when fewer labels are available in the training graph. For example, with label rates of $1.52\%$, $1.90\%$ and $3.04\%$ on Pubmed, the improvements of our framework combining with GCN are $4.48\%$, $3.30\%$ and $0.98\%$ respectively. This shows that the \ourmodel framework is especially useful when only a small number of labels are available in training, which is the common use case of GNNs.

\begin{table*}[h!]
\centering
\resizebox{\textwidth}{!}{%
\begin{tabular}{llrrrrrr}
 & & \multicolumn{3}{c}{Cora}    & \multicolumn{3}{c}{Pubmed} \\ 
 \cmidrule(lr){3-5}
 \cmidrule(lr){6-8} 
{\bf \# train labels}  & & \multicolumn{1}{c}{85 (3.21\%)} & \multicolumn{1}{c}{105 (3.88\%)} & \multicolumn{1}{c}{140 (5.17\%)} & \multicolumn{1}{c}{300 (1.52\%)} & \multicolumn{1}{c}{375 (1.90\%)}  &\multicolumn{1}{c}{600 (3.04\%)} \\ 
\cmidrule(lr){1-1}  \cmidrule(lr){3-5}
 \cmidrule(lr){6-8}
{\bf \% test labels}  & & \multicolumn{1}{c}{0\%} & \multicolumn{1}{c}{0\%} & \multicolumn{1}{c}{0\%} & \multicolumn{1}{c}{0\%} & \multicolumn{1}{c}{0\%}  &\multicolumn{1}{c}{0\%} \\
\cmidrule(lr){1-1}  \cmidrule(lr){3-5}
 \cmidrule(lr){6-8} 
Random & & 14.28 (0.00) & 14.28 (0.00)& 14.28 (0.00)& 33.33 (0.00)& 33.33 (0.00)& 33.33 (0.00)\\ 
\cmidrule{1-2}
\multirow{2}{*}{GCN} &  -    &   45.90 (3,26) & 47.54 (3.50) & 61.92 (1.50) & 52.68 (2.36) & 55.76 (3.32)  & 70.38 (2.31) \\
& \plusours &  \cellcolor{LightCyan} \textbf{+6.29 (1.49)} & \cellcolor{LightCyan}  \textbf{+5.20 (1.12)}  &\cellcolor{LightCyan} \textbf{+5.18 (0.66)} & \cellcolor{LightCyan}+4.48 (2.33) & \cellcolor{LightCyan}\textbf{+3.30(1.52)}  &\cellcolor{LightCyan} \textbf{+0.98(0.23)}  \\
\cmidrule{1-2}
\multirow{2}{*}{GS} & -      &  50.69 (1.50) & 56.24 (2.08) & 66.08 (0.96)& 59.34 (3.47)  & 64.37 (3.70)&  72.08 (1.87) \\
 & \plusours & \cellcolor{LightCyan} \textbf{+2.35 (0.56)} &\cellcolor{LightCyan} +\textbf{2.78 (0.59)} &\cellcolor{LightCyan} \textbf{+1.95 (0.45)} &\cellcolor{LightCyan} \textbf{+1.48 (0.41)} &\cellcolor{LightCyan} \textbf{+0.62 (0.21)}$^*$  &\cellcolor{LightCyan} \textbf{+0.65 (0.25)} \\ 
 \cmidrule{1-2}
\multirow{2}{*}{TK} & - &  63.74 (2.61) & 70.01 (1.93)  & 74.45 (0.34)  & 61.13 (5.03)  & 63.09 (5.57) & 75.46 (1.46) \\
 & \plusours & \cellcolor{LightCyan} \textbf{+0.96 (0.30)}$^*$ & \cellcolor{LightCyan}\textbf{+1.08 (0.37)}$^*$ & \cellcolor{LightCyan}\textbf{+0.30 (0.11)}$^*$ & \cellcolor{LightCyan}\textbf{+1.00 (0.21)}$^*$ & \cellcolor{LightCyan} \textbf{+1.34 (0.20)}  & \cellcolor{LightCyan}\textbf{+1.03 (0.22)}$^*$ \\
\cmidrule{1-2}
PL-EM$^\text{\cite{pfeiffer2015overcoming}}$& - & 20.70 (0.05) & 24.65 (0.38) & 30.46 (1.48) & 38.05 (4.85) & 44.85 (5.75) & 51.25 (3.06) \\
ICA$^\text{\cite{lu2003link}}$& - & 26.20 (0.51) & 41.05 (0.50) & 49.51 (1.90) & 44.40 (1.92) & 45.62 (0.86) & 54.26 (2.09)  \\
GMNN$^\text{\cite{qu2019gmnn}}$& -& 49.05 (1.86) & 54.55 (1.15) & 67.16 (1.86) & 58.03 (3.62) & 62.50 (3.77) & 71.03 (4.54) \\
 \bottomrule
\end{tabular}
}
\caption{Node classification accuracy with \textbf{unlabeled} test data varying number of \textbf{training labels} on Cora and Pubmed datasets. Numbers in bold represent significant improvement in a paired t-test at the $p < 0.05$ level, and numbers with $^*$ represent the best performing method in each column.}
\label{tab:results_unlabeled_learning_curve}
\vspace{-4mm}
\end{table*}

\begin{table*}[h!]
\centering
\resizebox{\textwidth}{!}{%
\begin{tabular}{llrrrrrr}
 & & \multicolumn{3}{c}{Cora}    & \multicolumn{3}{c}{Pubmed} \\ 
 \cmidrule(lr){3-5}
 \cmidrule(lr){6-8} 
{\bf \# train labels}  & & \multicolumn{1}{c}{85 (3.21\%)} & \multicolumn{1}{c}{105 (3.88\%)} & \multicolumn{1}{c}{140 (5.17\%)} & \multicolumn{1}{c}{300 (1.52\%)} & \multicolumn{1}{c}{375 (1.90\%)}  &\multicolumn{1}{c}{600 (3.04\%)} \\ 
\cmidrule(lr){1-1}  \cmidrule(lr){3-5}
 \cmidrule(lr){6-8} 
{\bf \% test labels}  & & \multicolumn{1}{c}{50\%} & \multicolumn{1}{c}{50\%} & \multicolumn{1}{c}{50\%} & \multicolumn{1}{c}{50\%} & \multicolumn{1}{c}{50\%}  &\multicolumn{1}{c}{50\%} \\
\cmidrule(lr){1-1}  \cmidrule(lr){3-5}
 \cmidrule(lr){6-8} 
Random & & 14.28 (0.00) & 14.28 (0.00)& 14.28 (0.00)& 33.33 (0.00)& 33.33 (0.00)& 33.33 (0.00)\\ 
\cmidrule{1-2}
\multirow{2}{*}{GCN} &  -    &   36.38 (1.35) & 48.31 (2.58) & 64.02 (1.54) & 54.11 (4.86) & 56.31 (3.10) & 68.13 (1.84) \\
& \plusours &  \cellcolor{LightCyan} \textbf{+15.69 (3.20)} & \cellcolor{LightCyan} {\bf +14.02 (3.38)}  & \cellcolor{LightCyan}{\bf +6.31 (0.89)} & \cellcolor{LightCyan}\textbf{+5.62 (1.17)} & \cellcolor{LightCyan}\textbf{+5.06 (3.24)}  & \cellcolor{LightCyan}\textbf{+ 4.60 (2.50)}  \\
\cmidrule{1-2}
\multirow{2}{*}{GS} & -      & 48.42 (2.82) & 57.52 (2.15)  & 65.04 (0.79) & 58.52 (5.42)  & 59.77 (4.68)   & 75.01 (4.86) \\
 & \plusours &\cellcolor{LightCyan} \textbf{+4.52 (0.84)} &\cellcolor{LightCyan} \textbf{+3.06 (0.20)} &\cellcolor{LightCyan} \textbf{+2.18 (0.21)} &\cellcolor{LightCyan} \textbf{+2.42 (0.27)} &\cellcolor{LightCyan} \textbf{+1.49 (0.10)}  &\cellcolor{LightCyan} \textbf{+2.67 (0.56)}  \\ 
 \cmidrule{1-2}
\multirow{2}{*}{TK} & - &  55.68 (2.08) & 61.51 (2.45)  & 67.95 (0.45)  & 63.05 (5.15) & 67.95 (0.45) & 74.01 (3.58) \\
 & \plusours & \cellcolor{LightCyan} {\bf +7.18 (1.88)}$^*$ & \cellcolor{LightCyan} \textbf{+3.04 (1.07)}$^*$ &\cellcolor{LightCyan} \textbf{+2.75 (0.47)}$^*$  & \cellcolor{LightCyan}{\bf +1.91 (0.75)}$^*$ & \cellcolor{LightCyan}+0.54 ~(0.44)$^*$ & \cellcolor{LightCyan}\textbf{+3.23 (0.78)}$^*$   \\
\cmidrule{1-2}
PL-EM$^\text{\cite{pfeiffer2015overcoming}}$& - & 20.35 (0.05) & 25.25 (0.35)& 31.45 (1.95)& 31.70 (4.78)& 34.92 (5.87)&  48.70 (5.72)\\
ICA$^\text{\cite{lu2003link}}$& - & 31.17 (3.66) & 42.07 (1.29) & 57.14 (1.81) & 33.38 (4.69) & 45.93 (5.48) & 46.97 (5.19)  \\
GMNN$^\text{\cite{qu2019gmnn}}$& -& 49.36 (2.22) & 56.58 (2.96) & 67.83 (1.91) & 62.16 (4.40) & 63.42 4.82) & 74.78 (3.63)  \\
 \bottomrule
\end{tabular}
}
\caption{Node classification accuracy with \textbf{partially-labeled} test data varying number of \textbf{training labels} on Cora and Pubmed datasets. Numbers in bold represent significant improvement in a paired t-test at the $p < 0.05$ level, and numbers with $^*$ represent the best performing method in each column.}
\label{tab:results_labeled_learning_curve}
\vspace{-4mm}
\end{table*}

\section{Ablation study} \label{sec: ablation}
\subsection{With or without predicted labels as input}
 \begin{figure}[h]
\centering
\subfloat[][Cora, GCN]{
    \includegraphics[height=1.5in]{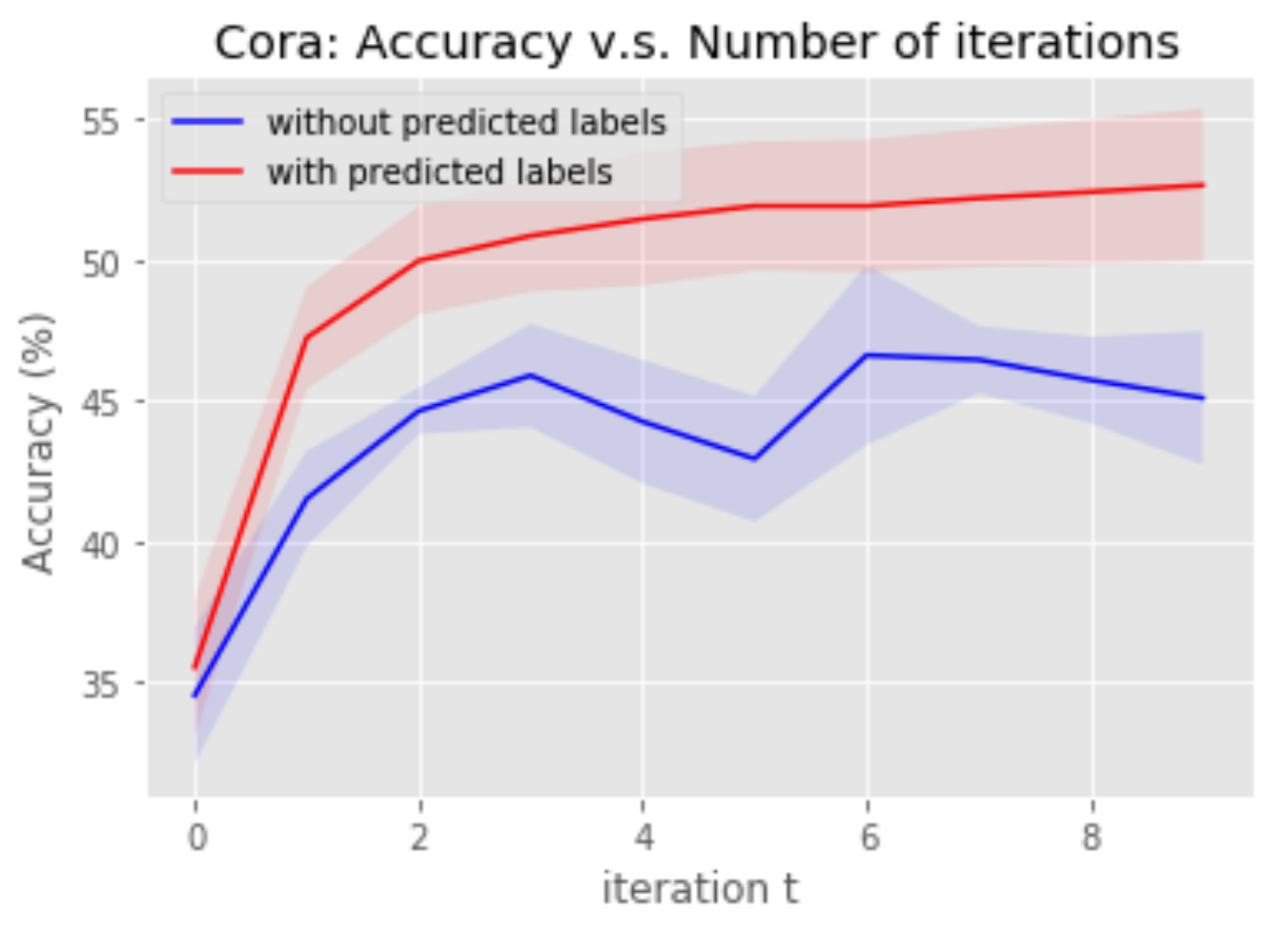}
    \label{fig: predict-gcn}
}
\vspace{-2.mm}
\subfloat[][Pubmed, TK]{
    \includegraphics[height=1.5in]{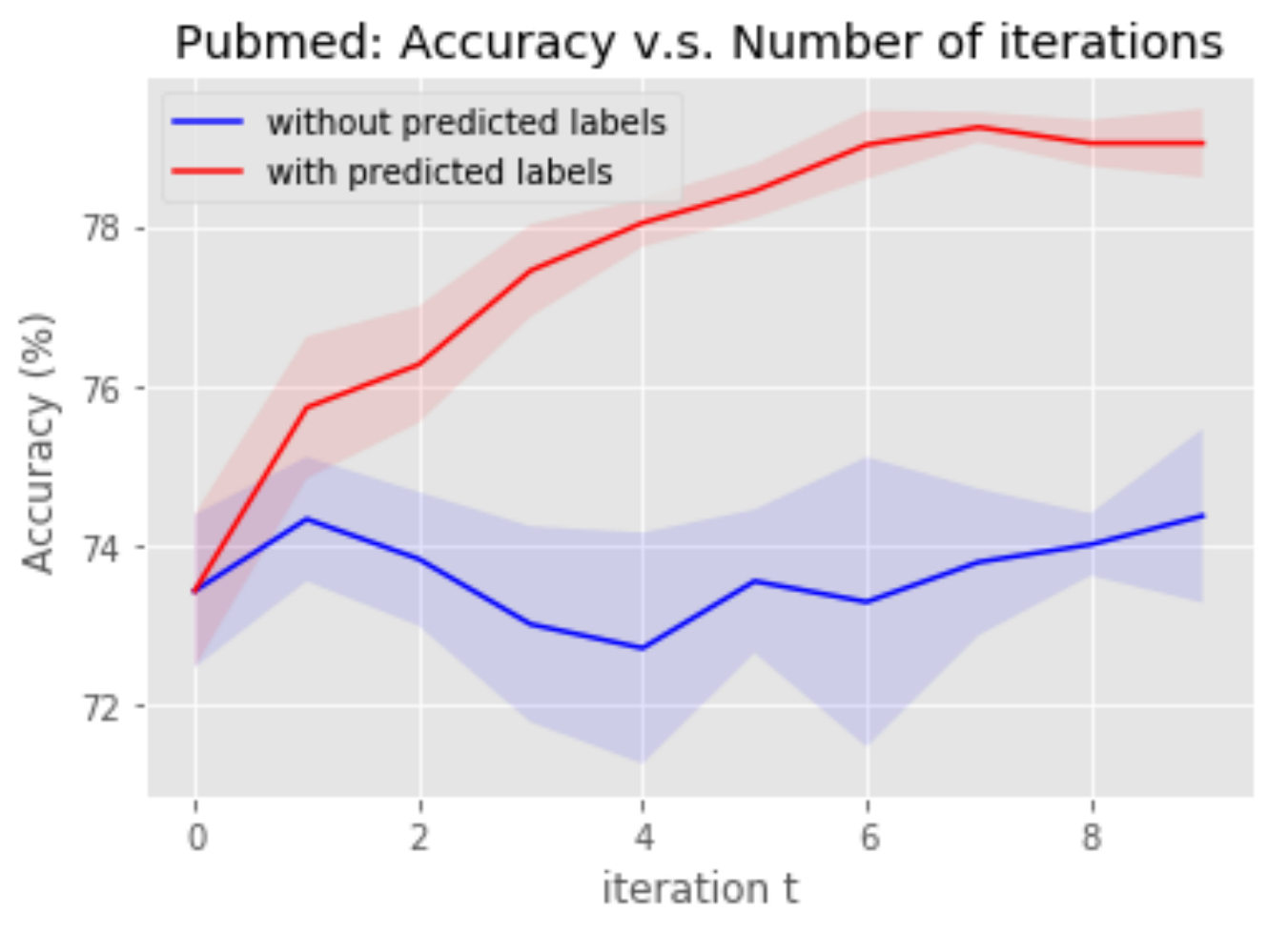}
    \label{fig: predict-tk}
}
\vspace{-1.mm}
\caption{\ourmodel performance with and without predicted labels on Cora and Pubmed. X-axis refers to iteration number $t$ in \Cref{sec:framework}}
\label{fig:no_pred}
\vspace{-4mm}
\end{figure}

 To investigate if adding predicted labels in model input adds extra information with partially-labeled test data, we tested the performance of a model variant which only use true labels as input with the same node masking procedure. \Cref{fig:no_pred} shows two examples on Cora with GCN (\Cref{fig: predict-gcn}) and Pubmed with TK (\Cref{fig: predict-tk}), where including predicted labels achieves better performance. We run the model 10 times and calculate the average and standard deviation (shown as shaded area) of classification accuracy at each iteration $t$ as described in \Cref{sec:framework}. We can see that adding predicted labels starts to improve the performance after the first iteration and achieves consistent gains.

\begin{table*}[h!]
\centering
\resizebox{\textwidth}{!}{ %
\begin{tabular}{llrrrrrrrrr}
 & & \multicolumn{3}{c}{Cora}    & \multicolumn{3}{c}{Pubmed} & \multicolumn{1}{c}{Friendster} & \multicolumn{1}{c}{Facebook} & \multicolumn{1}{c}{Protein}\\ 
 \cmidrule(lr){3-5}
 \cmidrule(lr){6-8} \cmidrule(lr){9-9}
 \cmidrule(lr){10-10} \cmidrule(lr){11-11}
{\bf \# labels}  & & \multicolumn{1}{c}{85 (3.21\%)} & \multicolumn{1}{c}{105 (3.88\%)} & \multicolumn{1}{c}{140 (5.17\%)} & \multicolumn{1}{c}{300 (1.52\%)} & \multicolumn{1}{c}{375 (1.90\%)}  &\multicolumn{1}{c}{600 (3.04\%)} & \multicolumn{1}{c}{641 (1.47\%)}           & \multicolumn{1}{c}{80 (1.76\%)} & \multicolumn{1}{c}{ 7607 (30\%)}     \\
\cmidrule(lr){1-2} 
\cmidrule(lr){3-5}
 \cmidrule(lr){6-8} 
 \cmidrule(lr){9-9}
 \cmidrule(lr){10-10}
 \cmidrule(lr){11-11}
Random & & 14.28 (0.00) & 14.28 (0.00)& 14.28 (0.00)& 33.33 (0.00)& 33.33 (0.00)& 33.33 (0.00)& 20.00 (0.00) & 50.00 (0.00) & 50.00 (0.00)\\ 
\cmidrule{1-2}
\multirow{2}{*}{GCN} &  -    & 45.15 (3.73) & 52.35 (2.01) &  65.11 (1.95) & 53.21 (4.04) & 57.15 (3.61)  & 70.81 (3.47)  &  29.80 (0.48)  &  65.89 (0.68)  & 73.03 (2.14)\\
 & \plusours-random & \cellcolor{LightCyan}  +0.02 (0.65)  &  \cellcolor{LightCyan} -1.83 (1.05)  &  \cellcolor{LightCyan}  +0.27 (0.18) &  \cellcolor{LightCyan} \textbf{+2.29 (0.34)}  &  \cellcolor{LightCyan} +1.35 (0.58)  &  \cellcolor{LightCyan} +1.05 (0.96)  &  \cellcolor{LightCyan}  -0.14 (0.40) &  \cellcolor{LightCyan}  \textbf{+1.16 (0.25)}  &  \cellcolor{LightCyan}  +0.16 (0.80)\\
\cmidrule{1-2}
\multirow{2}{*}{GS} & -      &46.38 (1.62) & 52.87 (1.03) & 63.46 (1.38) &  55.38 (3.48) &  57.61 (4.21) &  68.81 (4.15)  &  28.05 (0.56) &  65.20 (0.40) & 71.05 (0.40)\\
&\plusours-random &  \cellcolor{LightCyan}  -2.45 (0.22) &  \cellcolor{LightCyan} +0.46 (0.45) &  \cellcolor{LightCyan} -0.23 (0.67) & \cellcolor{LightCyan}  -0.02 (0.32)  &  \cellcolor{LightCyan}  +0.42 (0.31) & \cellcolor{LightCyan}  +0.34 (0.53)  & \cellcolor{LightCyan}  +0.21 (0.39)  &  \cellcolor{LightCyan} \textbf{+1.65 (0.15)} &  \cellcolor{LightCyan} +0.01 (0.22)\\ 
 \cmidrule{1-2}
\multirow{2}{*}{TK} & - & 61.99 (3.07) & 67.88 (1.80) &  73.04 (0.42) & 61.00 (4.93) &  61.91 (5.16) & 73.87 (3.99)  &  29.44 (0.39)  &  67.75 (0.40)  & 73.38 (0.57) \\
& \plusours-random &  \cellcolor{LightCyan}  -3.95 (1.08)  &  \cellcolor{LightCyan} -2.54 (0.63) &  \cellcolor{LightCyan} -2.28 (0.84) &  \cellcolor{LightCyan} -0.65 (0.56)& \cellcolor{LightCyan}  -0.78 (0.38)  &  \cellcolor{LightCyan}  -1.17 (0.78)  &  \cellcolor{LightCyan}  +0.26 (0.38)  &  \cellcolor{LightCyan} -0.05 (0.19)  &  \cellcolor{LightCyan} -0.13 (0.53) \\ \bottomrule
\end{tabular}
}
\caption{Node classification accuracy with \textbf{unlabeled} test data using \textbf{uniform sampling}. Numbers in bold represent significant improvement in a paired t-test at the $p < 0.05$ level.}
\label{tab:results_unlabeled_uniform}
\end{table*}


\subsection{Sampling from predicted labels or random ids}

Creating more expressive GNN representations by averaging out random features was first proposed by \citet{murphy2019relational}. 
\citet{murphy2019relational} shows a whole-graph classification application, Circulant Skip Links (CSL) graphs, where such randomized feature averaging is provably (and empirically) more expressive than GNNs.
Our Monte Carlo collective learning method can be seen as a type of feature averaging GNN representation though, unlike \citet{murphy2019relational}, the feature sampling is not at random, but rather driven by our own model recursively.
Hence, it is fair to ask if our performance gains are simply because random feature averaging is beneficial to GNN representations? 
Or does collective learning sampling actually improve performance?
We need an ablation study.

Therefore, in this section we investigate whether the gains of our method for unlabeled test data are from incorporating feature randomness, or from sampling w.r.t predicted labels (collective learning). To do so, we replace the samples drawn from previous prediction $\hat{\mY}$ as uniformly drawn from the set of class labels {\em at each gradient step} in \ourmodel. The results are shown in \Cref{tab:results_unlabeled_uniform}.
Clearly, the random features are not able to consistently improve the model performance as our method does (contrast \Cref{tab:results_unlabeled_uniform} with \Cref{tab:results_labeled} and \Cref{tab:results_unlabeled_learning_curve}). 
In summary, collective learning goes beyond the purely randomized approach of \citet{murphy2019relational}, providing much larger, statistically significant,  gains.

%% file: main.bbl
\begin{thebibliography}{36}
\providecommand{\natexlab}[1]{#1}
\providecommand{\url}[1]{\texttt{#1}}
\expandafter\ifx\csname urlstyle\endcsname\relax
  \providecommand{\doi}[1]{doi: #1}\else
  \providecommand{\doi}{doi: \begingroup \urlstyle{rm}\Url}\fi

\bibitem[Borgwardt et~al.(2005)Borgwardt, Ong, Sch{\"o}nauer, Vishwanathan,
  Smola, and Kriegel]{borgwardt2005protein}
Borgwardt, K.~M., Ong, C.~S., Sch{\"o}nauer, S., Vishwanathan, S., Smola,
  A.~J., and Kriegel, H.-P.
\newblock Protein function prediction via graph kernels.
\newblock \emph{Bioinformatics}, 21\penalty0 (suppl\_1):\penalty0 i47--i56,
  2005.

\bibitem[Chen et~al.(2019)Chen, Villar, Chen, and Bruna]{chen2019equivalence}
Chen, Z., Villar, S., Chen, L., and Bruna, J.
\newblock On the equivalence between graph isomorphism testing and function
  approximation with gnns.
\newblock In \emph{Advances in Neural Information Processing Systems}, pp.\
  15868--15876, 2019.

\bibitem[Devlin et~al.(2018)Devlin, Chang, Lee, and Toutanova]{devlin2018bert}
Devlin, J., Chang, M.-W., Lee, K., and Toutanova, K.
\newblock Bert: Pre-training of deep bidirectional transformers for language
  understanding.
\newblock \emph{arXiv preprint arXiv:1810.04805}, 2018.

\bibitem[Doersch et~al.(2015)Doersch, Gupta, and
  Efros]{doersch2015unsupervised}
Doersch, C., Gupta, A., and Efros, A.~A.
\newblock Unsupervised visual representation learning by context prediction.
\newblock In \emph{Proceedings of the IEEE International Conference on Computer
  Vision}, pp.\  1422--1430, 2015.

\bibitem[Fan \& Huang(2019)Fan and Huang]{fan2019recurrent}
Fan, S. and Huang, B.
\newblock Recurrent collective classification.
\newblock \emph{Knowledge and Information Systems}, 60\penalty0 (2):\penalty0
  741--755, 2019.

\bibitem[Gidaris et~al.(2019)Gidaris, Bursuc, Komodakis, P{\'e}rez, and
  Cord]{gidaris2019boosting}
Gidaris, S., Bursuc, A., Komodakis, N., P{\'e}rez, P., and Cord, M.
\newblock Boosting few-shot visual learning with self-supervision.
\newblock In \emph{Proceedings of the IEEE International Conference on Computer
  Vision}, pp.\  8059--8068, 2019.

\bibitem[Hamilton et~al.(2017)Hamilton, Ying, and
  Leskovec]{hamilton2017inductive}
Hamilton, W., Ying, Z., and Leskovec, J.
\newblock Inductive representation learning on large graphs.
\newblock In \emph{Advances in Neural Information Processing Systems}, pp.\
  1024--1034, 2017.

\bibitem[H{\'e}naff et~al.(2019)H{\'e}naff, Razavi, Doersch, Eslami, and
  Oord]{henaff2019data}
H{\'e}naff, O.~J., Razavi, A., Doersch, C., Eslami, S., and Oord, A. v.~d.
\newblock Data-efficient image recognition with contrastive predictive coding.
\newblock \emph{arXiv preprint arXiv:1905.09272}, 2019.

\bibitem[Jensen et~al.(2004)Jensen, Neville, and
  Gallagher]{jensen2004collective}
Jensen, D., Neville, J., and Gallagher, B.
\newblock Why collective inference improves relational classification.
\newblock In \emph{Proceedings of the tenth ACM SIGKDD international conference
  on Knowledge discovery and data mining}, pp.\  593--598, 2004.

\bibitem[Jia \& Benson(2020)Jia and Benson]{jia2020outcome}
Jia, J. and Benson, A.
\newblock Outcome correlation in graph neural network regression, 2020.

\bibitem[Kipf \& Welling(2016)Kipf and Welling]{kipf2016semi}
Kipf, T.~N. and Welling, M.
\newblock Semi-supervised classification with graph convolutional networks.
\newblock \emph{arXiv preprint arXiv:1609.02907}, 2016.

\bibitem[Koller et~al.(2007)Koller, Friedman, D{\v{z}}eroski, Sutton, McCallum,
  Pfeffer, Abbeel, Wong, Heckerman, Meek, et~al.]{koller2007introduction}
Koller, D., Friedman, N., D{\v{z}}eroski, S., Sutton, C., McCallum, A.,
  Pfeffer, A., Abbeel, P., Wong, M.-F., Heckerman, D., Meek, C., et~al.
\newblock \emph{Introduction to statistical relational learning}.
\newblock MIT press, 2007.

\bibitem[Lee et~al.(2017)Lee, Huang, Singh, and Yang]{lee2017unsupervised}
Lee, H.-Y., Huang, J.-B., Singh, M., and Yang, M.-H.
\newblock Unsupervised representation learning by sorting sequences.
\newblock In \emph{Proceedings of the IEEE International Conference on Computer
  Vision}, pp.\  667--676, 2017.

\bibitem[Lu \& Getoor(2003)Lu and Getoor]{lu2003link}
Lu, Q. and Getoor, L.
\newblock Link-based classification.
\newblock In \emph{Proceedings of the 20th International Conference on Machine
  Learning (ICML-03)}, pp.\  496--503, 2003.

\bibitem[Luan et~al.(2019)Luan, Zhao, Chang, and Precup]{luan2019break}
Luan, S., Zhao, M., Chang, X.-W., and Precup, D.
\newblock Break the ceiling: Stronger multi-scale deep graph convolutional
  networks.
\newblock \emph{arXiv preprint arXiv:1906.02174}, 2019.

\bibitem[Maron et~al.(2019)Maron, Fetaya, Segol, and
  Lipman]{maron2019universality}
Maron, H., Fetaya, E., Segol, N., and Lipman, Y.
\newblock On the universality of invariant networks.
\newblock \emph{arXiv preprint arXiv:1901.09342}, 2019.

\bibitem[Misra et~al.(2016)Misra, Zitnick, and Hebert]{misra2016shuffle}
Misra, I., Zitnick, C.~L., and Hebert, M.
\newblock Shuffle and learn: unsupervised learning using temporal order
  verification.
\newblock In \emph{European Conference on Computer Vision}, pp.\  527--544.
  Springer, 2016.

\bibitem[Moore \& Neville(2017)Moore and Neville]{moore2017deep}
Moore, J. and Neville, J.
\newblock Deep collective inference.
\newblock In \emph{Thirty-First AAAI Conference on Artificial Intelligence},
  2017.

\bibitem[Morris et~al.(2019)Morris, Ritzert, Fey, Hamilton, Lenssen, Rattan,
  and Grohe]{morris2019weisfeiler}
Morris, C., Ritzert, M., Fey, M., Hamilton, W.~L., Lenssen, J.~E., Rattan, G.,
  and Grohe, M.
\newblock Weisfeiler and leman go neural: Higher-order graph neural networks.
\newblock In \emph{Proceedings of the AAAI Conference on Artificial
  Intelligence}, volume~33, pp.\  4602--4609, 2019.

\bibitem[Murphy et~al.(2019)Murphy, Srinivasan, Rao, and
  Ribeiro]{murphy2019relational}
Murphy, R.~L., Srinivasan, B., Rao, V., and Ribeiro, B.
\newblock Relational pooling for graph representations.
\newblock \emph{arXiv preprint arXiv:1903.02541}, 2019.

\bibitem[Neville \& Jensen(2000)Neville and Jensen]{neville2000iterative}
Neville, J. and Jensen, D.
\newblock Iterative classification in relational data.
\newblock In \emph{Proc. AAAI-2000 workshop on learning statistical models from
  relational data}, pp.\  13--20, 2000.

\bibitem[Neville et~al.(2003{\natexlab{a}})Neville, Jensen, Friedland, and
  Hay]{neville2003learning}
Neville, J., Jensen, D., Friedland, L., and Hay, M.
\newblock Learning relational probability trees.
\newblock In \emph{Proceedings of the ninth ACM SIGKDD international conference
  on Knowledge discovery and data mining}, pp.\  625--630, 2003{\natexlab{a}}.

\bibitem[Neville et~al.(2003{\natexlab{b}})Neville, Jensen, and
  Gallagher]{neville2003simple}
Neville, J., Jensen, D., and Gallagher, B.
\newblock Simple estimators for relational bayesian classifiers.
\newblock In \emph{Third IEEE International Conference on Data Mining}, pp.\
  609--612. IEEE, 2003{\natexlab{b}}.

\bibitem[Noroozi \& Favaro(2016)Noroozi and Favaro]{noroozi2016unsupervised}
Noroozi, M. and Favaro, P.
\newblock Unsupervised learning of visual representations by solving jigsaw
  puzzles.
\newblock In \emph{European Conference on Computer Vision}, pp.\  69--84.
  Springer, 2016.

\bibitem[Pfeiffer~III et~al.(2015)Pfeiffer~III, Neville, and
  Bennett]{pfeiffer2015overcoming}
Pfeiffer~III, J.~J., Neville, J., and Bennett, P.~N.
\newblock Overcoming relational learning biases to accurately predict
  preferences in large scale networks.
\newblock In \emph{Proceedings of the 24th International Conference on World
  Wide Web}, pp.\  853--863, 2015.

\bibitem[Pham et~al.(2017)Pham, Tran, Phung, and Venkatesh]{pham2017column}
Pham, T., Tran, T., Phung, D., and Venkatesh, S.
\newblock Column networks for collective classification.
\newblock In \emph{Thirty-First AAAI Conference on Artificial Intelligence},
  2017.

\bibitem[Popescul et~al.(2002)Popescul, Ungar, Lawrence, and
  Pennock]{popescul2002towards}
Popescul, A., Ungar, L.~H., Lawrence, S., and Pennock, D.~M.
\newblock Towards structural logistic regression: Combining relational and
  statistical learning.
\newblock \emph{Departmental Papers (CIS)}, pp.\  134, 2002.

\bibitem[Qu et~al.(2019)Qu, Bengio, and Tang]{qu2019gmnn}
Qu, M., Bengio, Y., and Tang, J.
\newblock Gmnn: Graph markov neural networks.
\newblock \emph{arXiv preprint arXiv:1905.06214}, 2019.

\bibitem[Robbins \& Monro(1951)Robbins and Monro]{robbins1951}
Robbins, H. and Monro, S.
\newblock A stochastic approximation method.
\newblock \emph{Ann. Math. Statist.}, 22\penalty0 (3):\penalty0 400--407, 09
  1951.
\newblock \doi{10.1214/aoms/1177729586}.
\newblock URL \url{https://doi.org/10.1214/aoms/1177729586}.

\bibitem[Sen et~al.(2008)Sen, Namata, Bilgic, Getoor, Galligher, and
  Eliassi-Rad]{sen2008collective}
Sen, P., Namata, G., Bilgic, M., Getoor, L., Galligher, B., and Eliassi-Rad, T.
\newblock Collective classification in network data.
\newblock \emph{AI magazine}, 29\penalty0 (3):\penalty0 93--93, 2008.

\bibitem[Srinivasan \& Ribeiro(2019)Srinivasan and
  Ribeiro]{srinivasan2019equivalence}
Srinivasan, B. and Ribeiro, B.
\newblock On the equivalence between node embeddings and structural graph
  representations.
\newblock \emph{arXiv preprint arXiv:1910.00452}, 2019.

\bibitem[Teixeira et~al.(2019)Teixeira, Jalaian, and
  Ribeiro]{teixeira2019graph}
Teixeira, L., Jalaian, B., and Ribeiro, B.
\newblock Are graph neural networks miscalibrated?
\newblock \emph{arXiv preprint arXiv:1905.02296}, 2019.

\bibitem[Vijayan et~al.(2018)Vijayan, Chandak, Khapra, Parthasarathy, and
  Ravindran]{vijayan2018hopf}
Vijayan, P., Chandak, Y., Khapra, M.~M., Parthasarathy, S., and Ravindran, B.
\newblock Hopf: Higher order propagation framework for deep collective
  classification.
\newblock \emph{arXiv preprint arXiv:1805.12421}, 2018.

\bibitem[Xiang \& Neville(2008)Xiang and Neville]{xiang2008pseudolikelihood}
Xiang, R. and Neville, J.
\newblock Pseudolikelihood em for within-network relational learning.
\newblock In \emph{2008 Eighth IEEE International Conference on Data Mining},
  pp.\  1103--1108. IEEE, 2008.

\bibitem[Xu et~al.(2018)Xu, Hu, Leskovec, and Jegelka]{xu2018powerful}
Xu, K., Hu, W., Leskovec, J., and Jegelka, S.
\newblock How powerful are graph neural networks?
\newblock \emph{arXiv preprint arXiv:1810.00826}, 2018.

\bibitem[Yang et~al.(2017)Yang, Ribeiro, and Neville]{yang2017stochastic}
Yang, J., Ribeiro, B., and Neville, J.
\newblock Stochastic gradient descent for relational logistic regression via
  partial network crawls.
\newblock \emph{arXiv preprint arXiv:1707.07716}, 2017.

\end{thebibliography}
